\documentclass[11pt,a4paper]{dream2discern_preprint}

\usepackage{amsmath,amsfonts,bm}

\def\eqref#1{equation~\ref{#1}}

\def\1{\bm{1}}

\DeclareMathAlphabet{\mathsfit}{\encodingdefault}{\sfdefault}{m}{sl}
\SetMathAlphabet{\mathsfit}{bold}{\encodingdefault}{\sfdefault}{bx}{n}

\usepackage{booktabs,float,placeins,wrapfig,subcaption,amsthm,multirow,array}
\usepackage{algorithm,algpseudocode,pifont}
\usepackage{fontawesome5}
\usepackage[capitalize,noabbrev,nameinlink]{cleveref}
\renewcommand{\footnoterule}{\kern-3pt\hrule width\textwidth height0.45pt\kern2.55pt}

\definecolor{mydarkred}{rgb}{0.6,0,0}
\definecolor{myblue}{HTML}{268BD2}
\definecolor{thresholdgreen}{HTML}{7BC86C}
\definecolor{mygreen}{HTML}{658354}
\definecolor{blue}{HTML}{c9e9f6}
\definecolor{pink}{HTML}{F77189}
\definecolor{venuegray}{RGB}{120,120,120}

\definecolor{tablegray}{gray}{0.92}
\definecolor{sctshighlight}{RGB}{217,240,248}
\definecolor{baselinegray}{RGB}{240,240,240}
\newcommand{\method}{\textit{Dream2Discern}}

\newtheorem{proposition}{Proposition}

\preprintheadmark{\includegraphics[height=21pt]{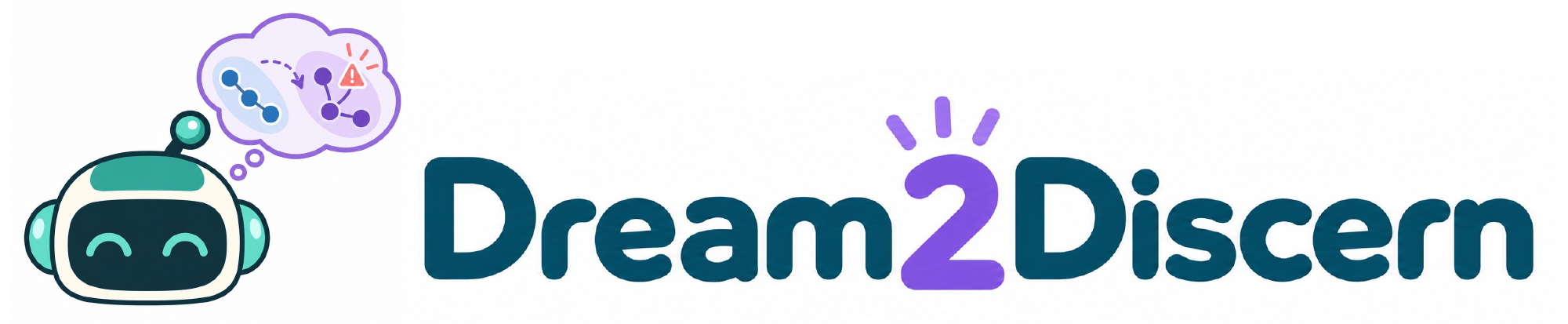}}
\title{Dream\textcolor{d2dDream}{2}Discern:
Learning from Hard Latent
Variants for Runtime Failure
Prediction in Multi-Agent Systems}
\author{Boxuan Zhang$^{1}$ \quad Jianing Zhu$^{2}$ \quad Xiaotian Han$^{3}$ \quad Ruixiang Tang$^{1,\dagger}$}
\preprintaffiliations{%
  \mbox{$^{1}$\,\raisebox{\dimexpr.65ex-\height/2\relax}{\includegraphics[height=2.15ex]{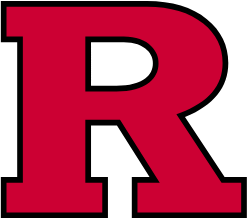}}\ \,Rutgers University}%
  \hspace{1.1em}%
  \mbox{$^{2}$\,\raisebox{\dimexpr.65ex-\height/2\relax}{\includegraphics[trim=221 269 1064 212,clip,height=2.4ex]{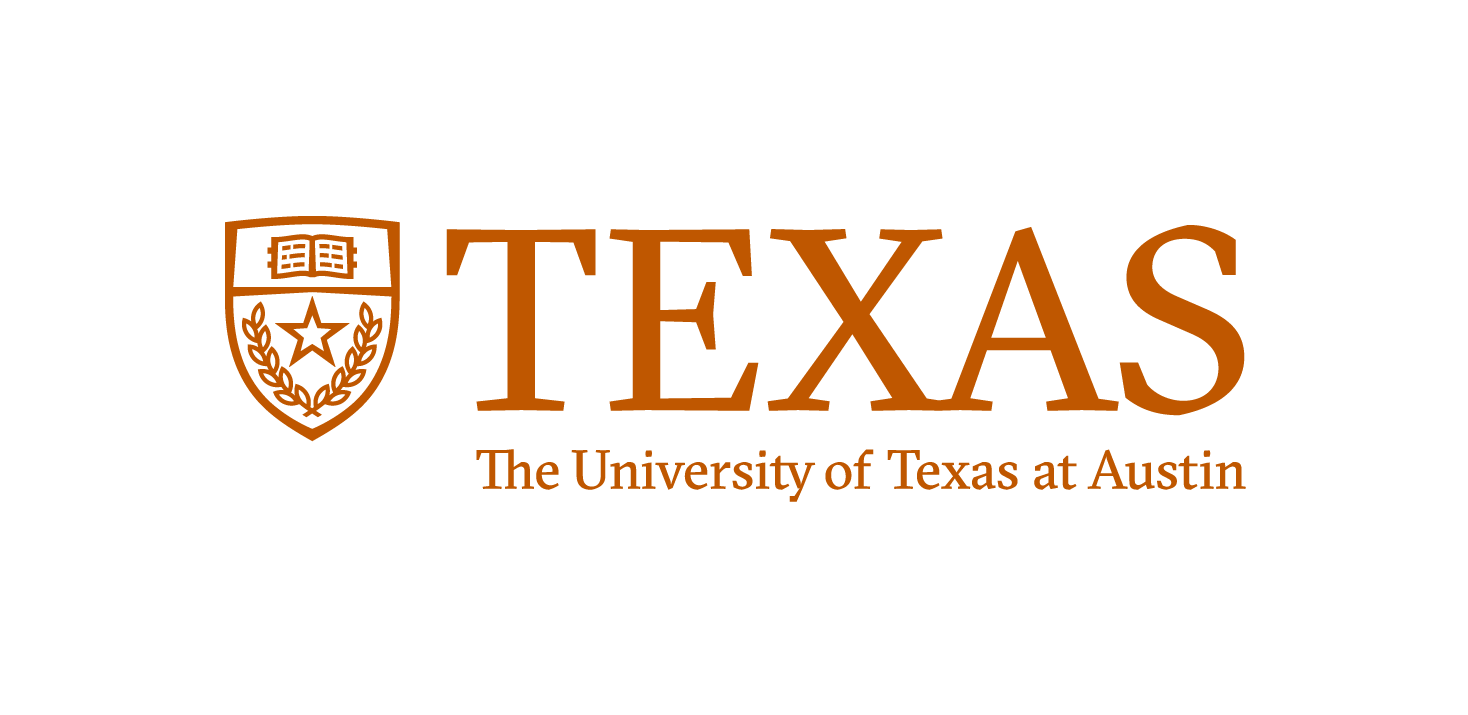}}\ \,The University of Texas at Austin}%
  \hspace{1.1em}%
  \mbox{$^{3}$\hspace{0.35em}\raisebox{\dimexpr.3ex-\height/2\relax}{\includegraphics[height=2.4ex]{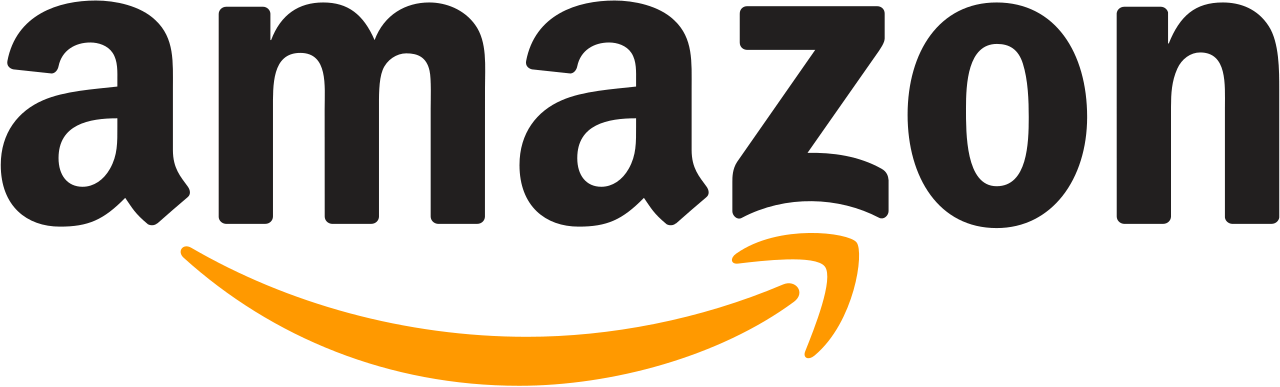}}}}
\preprintlinks{%
  {\color{d2dTeal}\faGlobe}\ \textbf{Project Page:}\ \textcolor{d2dBlue}{\texttt{https://zbox1005.github.io/dream2discern-project}}\par
  {\color{d2dTeal}\faGithub}\ \textbf{Code:}\ \textcolor{d2dBlue}{\texttt{https://github.com/ZBox1005/Dream2Discern}}\par
  {\color{d2dTeal}\faEnvelope}\ \textbf{Contact:}\ \texttt{\{boxuan.zhang, ruixiang.tang\}@rutgers.edu}}
\preprintstatus{Research preprint}
\date{September 2026}
\hypersetup{pdftitle={Dream2Discern: Learning from Hard Latent Variants for Runtime Failure Prediction in Multi-Agent Systems},
  pdfauthor={Boxuan Zhang, Jianing Zhu, Xiaotian Han, Ruixiang Tang}}

\begin{document}
\makepreprinttitle
\begin{preprintfrontmatter}
\makepreprintauthors

\begin{abstract}
Reliable deployment of LLM-based multi-agent systems calls for runtime auditors that predict terminal outcomes from observed execution prefixes.
These auditors typically learn from a limited corpus of trajectories with verified outcomes.
Expanding this corpus requires costly execution and verification, yet broader failure coverage alone provides no explicit criterion for identifying prefixes that the current auditor finds difficult to classify.
To this end, we propose \textbf{\textit{Dream2Discern}}, a runtime auditor trained through implicit augmentation guided by its own prediction loss.
During \textit{Dream}, we select perturbations of the auditor's representation operators to increase prediction loss on recorded prefixes, yielding \emph{hard latent variants}.
During \textit{Discern}, the auditor learns to predict execution outcomes from these variants using the original labels.
The variants are regenerated as the auditor learns, turning a limited corpus into a renewable source of hard examples and requiring no additional trajectory collection or outcome verification.
Extensive evaluations on \textsc{AFTraj-2K} and StepShield show that \textit{Dream2Discern} outperforms competitive failure auditors and achieves over 90\% trajectory-level weighted accuracy on both benchmarks while balancing prediction accuracy and timeliness. 
\end{abstract}
\makepreprintlinks
\end{preprintfrontmatter}
{\renewcommand{\thefootnote}{\fnsymbol{footnote}}\footnotetext[2]{\emph{Corresponding author}}}

\section{Introduction}
\label{sec:introduction}

LLM-based multi-agent systems coordinate specialized agents and tools to solve complex tasks, from software development to open-ended web interaction~\citep{wu2024autogen,hong2023metagpt,fourney2024magentic}.
However, the complexity of these tasks and their execution environments makes agents prone to reasoning and tool-use errors that can potentially compromise task completion~\citep{cemri2026multi,zhang2025agent}.
In open-world environments, these errors can lead agents to make unnecessary tool calls or modify external systems in unintended ways~\citep{li2023deepinception,ruan2024identifying,oh2026uncertainty,zhang2025dive}.
Accordingly, reliable deployment calls for \emph{runtime failure prediction}, where an auditor assesses the risk of eventual task failure from the execution history observed so far to inform timely intervention~\citep{zhang2026agentforesight,huang2026prefixguard}.

Despite its importance, existing learning-based runtime auditors, such as TRACES~\citep{li2026traces} and AgentForesight~\citep{zhang2026agentforesight}, typically rely on a limited corpus of carefully curated trajectories with verified outcomes, as illustrated in Figure~\ref{fig:scts-motivation}(a).
A straightforward way to improve these auditors is to augment the training corpus to cover a broader range of failure scenarios, either by collecting new trajectories or by editing existing ones (Figure~\ref{fig:scts-motivation}(b)).
However, collecting new trajectories requires agent rollouts and outcome verification, making corpus expansion costly to scale~\citep{xie2026agentsynth,xu2025agenttrek}.
Editing recorded trajectories reuses the observed history, but changing an action or tool response can alter the continuation and invalidate the original outcome label, requiring resumed execution and re-verification~\citep{kong2026aegis,bonagiri2026causalflow}.
More fundamentally, broader failure coverage alone provides no explicit criterion for identifying trajectories whose prefixes the current auditor finds difficult to assess.
These considerations naturally motivate the following research question:
\begin{preprintquestion}
{\color{d2dTeal}\small\bfseries RESEARCH QUESTION}\par\smallskip
\emph{\textbf{Can we identify and exploit the current auditor's hard cases without acquiring additional labeled trajectories?}}
\end{preprintquestion}

To address this question, we propose an implicit augmentation approach that constructs hard latent variants within the auditor (Figure~\ref{fig:scts-motivation}(c)).
\textbf{Our key insight is to use the auditor's own prediction loss as a criterion for deciding which variants to construct.}
We perturb the auditor's representation computation to increase prediction loss on recorded prefixes, yielding \emph{hard latent variants} for training with the original outcome labels.
To realize this idea, we develop \textbf{\method{}}, a runtime auditor that alternates between two training steps: constructing hard latent variants (\emph{Dream}) and learning from them to discern execution outcomes (\emph{Discern}).
Specifically, during \emph{\textbf{Dream}}, we search for bounded, low-rank perturbations to the auditor's query and value projection weights that increase batch prediction loss on recorded prefixes.
We keep the prediction head fixed during this search, thus restricting the induced difficulty to changes in prefix representations. 
In \emph{\textbf{Discern}}, we train the auditor to distinguish success from failure under the selected perturbation. 
At runtime, the trained auditor operates without perturbations, assigning a failure score to each observed prefix and applying a calibrated threshold to predict the terminal outcome. 
Together, the \emph{Dream} and \emph{Discern} steps alternate throughout training, allowing the constructed variants to evolve with the auditor.


\begin{figure}[t]
    \centering
    \includegraphics[width=\linewidth]{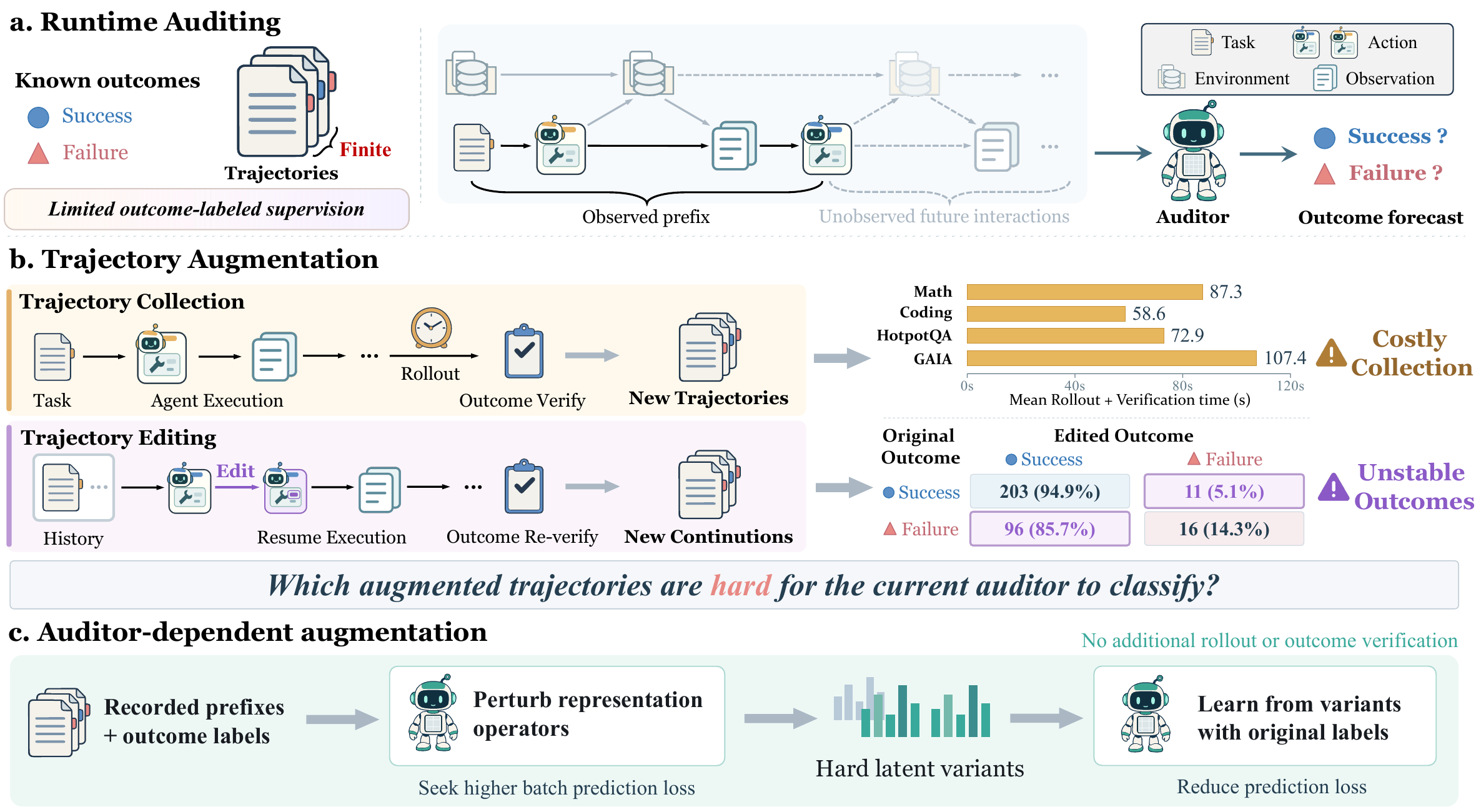}
    \caption{\textbf{Motivation of \method{}.}
    \textbf{(a)} A runtime auditor predicts terminal outcomes from observed prefixes but learns from a limited labeled corpus.
    \textbf{(b)} Trajectory augmentation seeks broader failure coverage, but collection is costly and editing can invalidate outcome labels (right).
    Broader coverage alone does not identify prefixes that the current auditor finds hard to classify.
    \textbf{(c)} We construct hard latent variants through operator perturbations chosen to increase batch prediction loss, then train on these variants while keeping recorded prefixes and labels unchanged.}
    \label{fig:scts-motivation}
\end{figure}

We conduct extensive runtime evaluations of \textit{Dream2Discern} on \textsc{AFTraj-2K}~\citep{zhang2026agentforesight} and StepShield~\citep{felicia2026stepshield}.
Our method achieves over $\mathbf{90\%}$ trajectory-level weighted accuracy (wACC) on both benchmarks and outperforms the strongest auditor baseline by up to $\mathbf{11.11}\%$ wACC, while maintaining a practical accuracy--timeliness trade-off.
These gains extend to Llama backbones, and loss-guided variants improve outcome discrimination over training without perturbations or with random perturbations.
Our main contributions can be summarized as follows:




\begin{itemize}
    \item 
    We show that augmenting the training corpus with edited prefixes does not improve held-out prediction under a fixed training budget. This motivates constructing hard latent variants tailored to the current auditor (Section~\ref{sec:prefix-auditor}).

    \item We develop \textbf{\textit{Dream2Discern}}, a parameter-efficient runtime auditor with alternating \emph{Dream} and \emph{Discern} updates.
    \emph{Dream} constructs hard latent variants through loss-guided perturbations of the auditor's representation computation, and \emph{Discern} learns to predict execution outcomes from these variants using the recorded labels (Section~\ref{sec:steer-to-confuse} and Section~\ref{sec:train-to-separate}).

    \item We demonstrate improved prediction performance on \textsc{AFTraj-2K} and StepShield across different model families.
    Ablations and further analyses support the effectiveness of loss-guided variant construction and its contribution to runtime auditor learning (Section~\ref{sec:exp}).
\end{itemize}

\section{Problem Formulation}
\label{sec:problem}


\noindent \textbf{Multi-Agent Trajectory.}
Consider a system of LLM agents that collaborate on a task through communication and tool use. For execution $i$, let $c_i$ denote the initial task context available to the auditor. We write the complete trajectory as $\tau_i=(c_i,u_{i,1},\ldots,u_{i,T_i})$,
where $T_i$ is the number of interaction steps. Each record $u_{i,t}$ identifies its source (agent, user, tool, environment) and records the associated message, action, or observation.  
Each completed trajectory has a terminal outcome $Y_i\in\{0,1\}$, where $Y_i=1$ denotes failure and $Y_i=0$ denotes success.


\noindent \textbf{Runtime Failure Prediction.}
Following prior work on runtime monitoring~\citep{zhang2026agentforesight,li2026traces,park2026hide}, we consider an auditor that observes only the prefix $\tau_{i,\leq t}=(c_i,u_{i,1},\ldots,u_{i,t})$ at each nonterminal step $1\leq t<T_i$.
The auditor predicts the terminal outcome $Y_i$ from this prefix, with its prediction denoted by $\widehat{Y}_{i,t}\in\{0,1\}$.
As illustrated in Figure~\ref{fig:scts-motivation}(a), the auditor updates its prediction after each newly observed interaction.
Future interactions, the realized terminal length $T_i$, and the outcome $Y_i$ are unavailable to the auditor at prediction time.

\noindent \textbf{Learning Setup.}
Let $f_\Theta$ denote an auditor with trainable parameters $\Theta$ and a sigmoid output. Given an input $X_{i,t}$ constructed from the observed prefix $\tau_{i,\leq t}$, it produces a scalar failure score
\begin{equation}
    s_{i,t}=f_\Theta(X_{i,t})\in[0,1].
    \label{eq:auditor-parameterization}
\end{equation}
Thresholding this score yields $\widehat{Y}_{i,t}$, with the decision rule specified in Section~\ref{sec:runtime-decision}.
For the learning objective, let $q=(i,t)$ index a recorded prefix with input $X_q=X_{i,t}$ and the corresponding trajectory's terminal outcome $Y_q=Y_i$ as its prediction target. 
For an indexed batch $\mathcal B=(q_b)_{b=1}^{n}$ of selected prefixes, we define the mean BCE loss as $\mathcal L_{\mathcal B}(\Theta)=\frac{1}{n}\sum_{b=1}^{n} \ell_{\mathrm{BCE}}\!\left(f_\Theta(X_{q_b}),Y_{q_b}\right)$.

\section{Method}
\label{sec:method}

We present \textbf{\textit{Dream2Discern}}, a runtime auditor trained through alternating \emph{Dream} and \emph{Discern} steps.
We first analyze the limitations of trajectory augmentation and motivate prediction loss as a criterion for constructing hard latent variants. 
We then describe \emph{Dream}, which constructs hard latent variants by perturbing the auditor's adapted representation path. 
The \emph{Discern} step trains the auditor under the selected perturbation. 
Finally, we describe how prefix failure scores yield runtime decisions. 

\subsection{Motivation}
\label{sec:prefix-auditor}


Reliable deployment of LLM-based multi-agent systems calls for runtime failure prediction to support timely intervention.
Existing learning-based auditors typically learn from a limited corpus of carefully curated trajectories with verified outcomes.
To improve prediction, a straightforward approach is to augment this corpus to cover a broader range of failure scenarios~\citep{xu2025agenttrek,kong2026aegis}.
However, the practical value of this approach depends on both the cost of acquiring additional labeled trajectories and whether they improve held-out prediction.

To this end, we first assess the cost of acquiring additional labeled trajectories.
We measure the combined rollout and outcome-verification time for $\sim$2K trajectories across task domains in \textsc{AFTraj-2K}~\citep{zhang2026agentforesight} (see the right panel of Figure~\ref{fig:scts-motivation}(b)).
Results show that averaged acquisition times reach tens of seconds per trajectory, making repeated corpus expansion time-consuming and unsustainable.
Editing offers an alternative by reusing recorded history, but the resulting continuation may no longer share the original outcome label.
To examine this issue, we modify an action or tool response and resume execution, obtaining 326 verified continuations in \textsc{AFTraj-2K} (see Appendix~\ref{app:trajectory-editing} for detailed settings).
Among verified continuations from failed source trajectories, 85.7\% succeed, showing that source labels cannot simply be reused after editing and re-execution.

Beyond acquisition cost, we examine whether the verified edits improve auditor learning.
We first compare paired original and edited prefix representations extracted by the same frozen LLM encoder.
Despite their textual differences, the two groups substantially overlap in a shared PCA projection (Figure~\ref{fig:augmentation-diagnostics}(a)).
To evaluate their training utility, we train prediction heads on original prefixes alone (\emph{Original}) or together with the verified edits (\emph{+Edited}), keeping the encoder and optimization budget fixed.
We evaluate both heads on the same held-out prefixes using class-balanced BCE, where lower loss indicates better prediction.
The higher loss of \emph{+Edited} relative to \emph{Original} (Figure~\ref{fig:augmentation-diagnostics}(b)) shows that these edits do not improve held-out prediction under this budget.

\begin{wrapfigure}{r}{0.65\textwidth}
    \centering
    \begin{minipage}{0.48\linewidth}
        \centering
        \includegraphics[width=\linewidth]{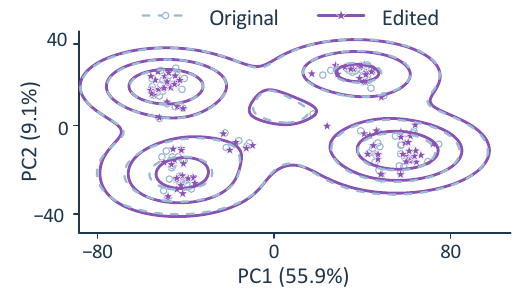}
        {\small\textbf{(a) Representation overlap}}
    \end{minipage}
    \begin{minipage}{0.48\linewidth}
        \centering
        \includegraphics[width=\linewidth]{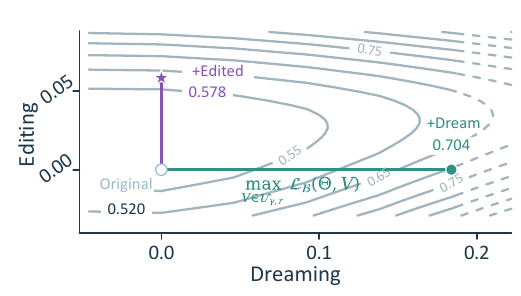}
        {\small\textbf{(b) Held-out prediction loss}}
    \end{minipage}
    \caption{\textbf{(a)} Original and edited prefixes overlap in a shared PCA projection.
    \textbf{(b)} Training with edits does not reduce held-out class-balanced BCE. A temporary \textit{Dream} perturbation selected on training prefixes raises this loss with the \emph{Original} head fixed.}
    \label{fig:augmentation-diagnostics}
\vspace{-4mm}
\end{wrapfigure}


These observations motivate us to construct hard latent variants that target the auditor's prediction difficulties.
We propose to seek bounded weight perturbations that alter prefix representations toward higher batch prediction loss, while preserving the recorded prefixes and outcome labels.
To examine this intuition, we select a perturbation using only training prefixes and temporarily apply it to the encoder, keeping the \emph{Original} prediction head fixed.
On the same held-out prefixes, the perturbed auditor (\emph{+Dream}) has higher BCE than \emph{Original} (Figure~\ref{fig:augmentation-diagnostics}(b)).
This result shows that the loss-guided perturbation also makes outcome prediction harder beyond the batch used to select it, without acquiring additional labeled trajectories.
Experimental setup details are in Appendix~\ref{app:augmentation-diagnostics}.

\subsection{Dream: Constructing Hard Latent Variants}
\label{sec:steer-to-confuse}

We instantiate $f_\Theta$ with a pretrained LLM (e.g., Qwen3-8B) whose weights $\phi_0$ remain frozen.
We apply LoRA~\citep{hu2021lora} to the query and value projections in Transformer blocks following a fixed intermediate block, and denote the adapter parameters by $a$.
Let $h_{\phi_0,a}(X)$ denote the final-layer hidden state at the \textit{last token} of prefix input $X$.
A classification head $g_\psi$ maps this state to a scalar logit, yielding $f_\Theta(X)=\sigma(g_\psi(h_{\phi_0,a}(X)))$, where $\sigma$ is the sigmoid function and $\Theta=(a,\psi)$.


Each training batch $\mathcal B=(q_b)_{b=1}^{2m}$ contains $m$ prefixes from each outcome class, so $n=2m$ in the batch BCE loss (Figure~\ref{fig:scts-overview}, \ding{172}).
In the \textit{Dream} step, we hold $\Theta$ fixed and seek a shared perturbation that increases prediction loss on this batch.
The perturbation acts only on the effective weights of the adapted query and value projections, so the search changes prefix representations under the fixed prediction head.
Let $\mathcal S$ index these projections and $V=(V_j)_{j\in\mathcal S}$ denote the perturbation shared across the batch.
The effective projection weights and their temporary perturbations are
\begin{equation}
    W_j(a)=W_j^0+\kappa_L B_jA_j,\qquad
    W_j(a)\;\mapsto\;W_j(a)+V_j,\qquad j\in\mathcal S,
    \label{eq:lora}
\end{equation}
where $W_j^0$ is frozen, $A_j$ and $B_j$ are the trainable LoRA factors collected in $a$, and $\kappa_L$ is the adapter scaling.
Since $B_jA_j=(cB_j)(A_j/c)$ for any $c\neq0$, an operator-level perturbation budget is invariant to this rescaling of the LoRA factors. We write $\mathcal L_{\mathcal B}(\Theta,V)$ for the batch loss in Section~\ref{sec:problem} computed with these perturbed operators, with $\mathcal L_{\mathcal B}(\Theta,0)=\mathcal L_{\mathcal B}(\Theta)$. Throughout the \textit{Dream} step, the recorded prefix inputs and their outcome labels remain unchanged.


\noindent \textbf{Bounded, low-rank perturbations.}
To identify prediction difficulties under small, structured changes to the auditor's operators, we bound the relative perturbation magnitudes and the rank at each projection.
Using frozen projection norms $\omega_j:=\lVert W_j^0\rVert_F>0$ as reference scales, we define:
\begin{equation}
    \lVert V\rVert_{\omega}^{2}
    :=\sum_{j\in\mathcal S}\frac{\lVert V_j\rVert_F^2}{\omega_j^2}, \quad
    \mathcal U_{\gamma,r}
    :=\big\{V:\ \lVert V\rVert_{\omega}\leq\gamma,\quad
    \operatorname{rank}(V_j)\leq r\ \ \forall j\in\mathcal S\big\},
    \label{eq:ambiguity-set}
\end{equation}
where $\gamma>0$ is a global relative radius and $r\geq1$ is an integer rank cap. 
The scales normalize perturbation magnitudes across projections, while the single budget bounds their combined relative size. We keep these scales fixed so that adapter updates do not change the reference used to measure perturbation size.
The rank cap limits each projection's direct output displacement up to $r$ directions.

\begin{figure}[t]
    \centering
    \includegraphics[width=0.99\linewidth]{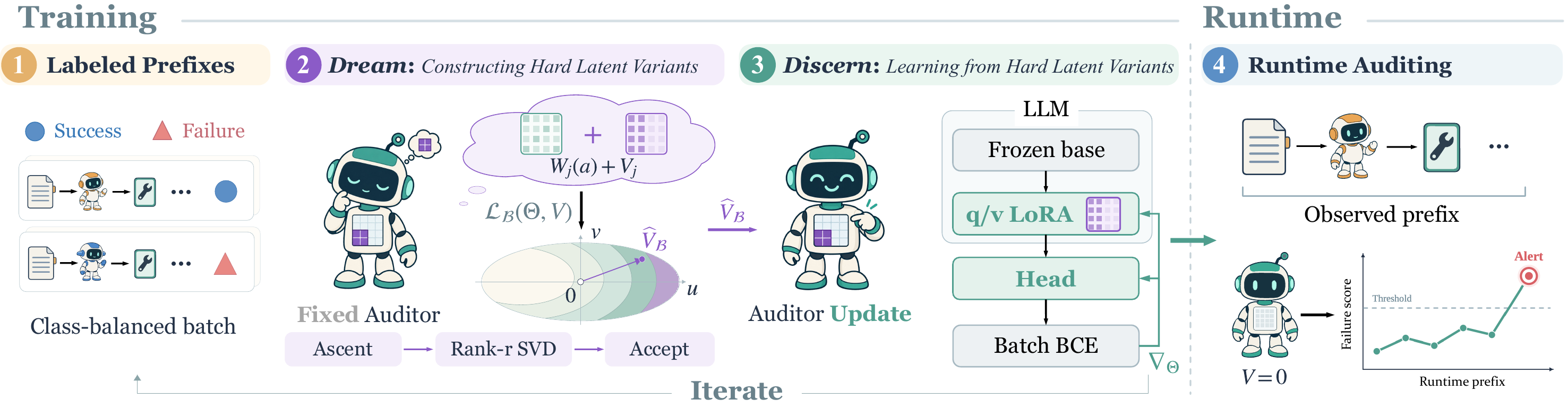}
    \caption{\textbf{Overview of \method{}.}
    \textbf{\ding{172} \textit{Labeled prefixes}.}
    We form a class-balanced batch from recorded success and failure prefixes, keeping their inputs and outcome labels unchanged throughout the training steps.
    \textbf{\ding{173} \textit{Dream}.}
    With the auditor fixed, we search for shared, bounded, low-rank perturbations to its query and value projection weights that increase batch prediction loss.
    \textbf{\ding{174} \textit{Discern}.}
    Holding the selected perturbation fixed, we update the LoRA adapters and prediction head to minimize the loss on the resulting hard latent variants.
    \textbf{\ding{175} \textit{Runtime auditing}.}
    We discard the perturbation and apply a calibrated threshold to each prefix score to predict failure during execution.}
    \label{fig:scts-overview}
\end{figure}

\noindent \textbf{From operator perturbations to hard latent variants.}
Within the perturbation set above, we define the ideal \textit{Dream} perturbation by maximizing batch prediction loss:
\begin{equation}
    V_{\mathcal B}^{\star}(\Theta)
    \in\arg\max_{V\in\mathcal U_{\gamma,r}}
      \mathcal L_{\mathcal B}(\Theta,V).
    \label{eq:ideal-inner}
\end{equation}
For an individual prefix, BCE increases when the predicted failure score rises for a successful trajectory or falls for a failed one.
The objective gives both outcome classes equal weight and evaluates a single perturbation on the batch.
Under the fixed prediction head, \textit{Dream} seeks latent variants of the recorded prefixes with higher batch prediction loss.
Although the operator perturbation is shared, its effect depends on the incoming token states of each prefix~\citep{wang2023out}.
For a fixed $V$, let $H_{q,j}^{V}$ collect the token states entering projection $j$ for prefix $q$ in the perturbed forward pass.
With the compact SVD $V_j=P_j\Sigma_jQ_j^\top$, the direct additive contribution of $V_j$ at these incoming states is
\begin{equation}
    H_{q,j}^{V}V_j^{\top}
    =(H_{q,j}^{V}Q_j)\,\Sigma_j\,P_j^{\top}.
    \label{eq:implicit-transformation}
\end{equation}
The columns of $P_j$ define at most $r$ output directions shared across prefixes, while $H_{q,j}^{V}Q_j\Sigma_j$ determines their prefix-dependent coefficients.
Query perturbations can alter attention weights over prefix tokens, while value perturbations can alter the vectors aggregated by attention.
These local changes propagate through the remaining Transformer computation to the final prefix representations.
The final prefix representations induced by an accepted \textit{Dream} perturbation form \emph{hard latent variants}, with hardness measured by the increase in batch prediction loss.
\textit{Discern} then learns to predict the recorded outcomes from these variants.
We further observe in Section~\ref{exp:mecha} that held-out prefixes with larger absolute score changes under training-selected \textit{Dream} perturbations are more often misclassified by the unperturbed auditor.

\noindent \textbf{Solving for the perturbation.}
To construct hard latent variants, we seek operator changes within $\mathcal U_{\gamma,r}$ that maximize the first-order increase in batch prediction loss.
Hence, we linearize the objective in Equation~\ref{eq:ideal-inner} at the unperturbed auditor with $V=0$.
Let $G_j=\left.\nabla_{V_j}\mathcal L_{\mathcal B}(\Theta,V)\right|_{V=0}$ denote the loss gradient for projection $j$.
The first-order loss increase is $\sum_{j\in\mathcal S}\langle G_j,V_j\rangle$, where $\langle A,B\rangle=\operatorname{tr}(A^\top B)$ is the Frobenius inner product.
In the scaled coordinates of Equation~\ref{eq:ambiguity-set}, the gradient is $C_j=\omega_jG_j$.
Let $C_{j,r}$ be a best approximation of $C_j$ of rank at most $r$ in Frobenius norm, obtained by SVD truncation.
Define $g_r=(\sum_{j\in\mathcal S}\lVert C_{j,r}\rVert_F^2)^{1/2}$ as the norm of the retained scaled gradients.
\begin{preprintproposition}
\begin{proposition}[Optimal linearized perturbation]
\label{prop:linear-oracle}
For fixed $\Theta$ and $\mathcal B$, the maximum of $\sum_{j\in\mathcal S}\langle G_j,V_j\rangle$ over $\mathcal U_{\gamma,r}$ is $\gamma g_r$.
If $g_r>0$, one maximizer is
\begin{equation}
    V^{\mathrm{lin}}_j=\gamma\,\omega_j\,\frac{C_{j,r}}{g_r},\qquad j\in\mathcal S.
    \label{eq:linear-oracle}
\end{equation}
If $g_r=0$, the maximum is attained at $V=0$.
\end{proposition}
\end{preprintproposition}
The proof is given in Appendix~\ref{app:proof-linear-oracle}.
The solution retains the leading singular components of each scaled gradient and allocates relative perturbation magnitudes in proportion to their retained Frobenius norms.
In practice, we normalize the full scaled gradient before SVD truncation (Figure~\ref{fig:scts-overview}, \ding{173}).
For $g_r>0$, the resulting candidate $\overline V_{\mathcal B}$ follows one such maximizing direction, with relative radius at most $\gamma$, while it is zero for $g_r=0$.
To check whether the candidate increases the nonlinear batch loss, we start at $\overline V_{\mathcal B}$ and successively halve it, accepting the first tested point whose measured loss strictly exceeds the loss at $V=0$.
We test at most six points, as specified in Algorithm~\ref{alg:dream2discern} of Appendix~\ref{app:solver}.
If no tested step is accepted, we return $V=0$.
Let $\lambda_{\mathcal B}\in\{0,1,2^{-1},\ldots,2^{-5}\}$ denote the returned step fraction, with $0$ indicating rejection. The returned perturbation and realized radius:
\begin{equation}
    \widehat V_{\mathcal B}=\lambda_{\mathcal B}\overline V_{\mathcal B},
    \qquad
    \rho_{\mathcal B}:=\lVert\widehat V_{\mathcal B}\rVert_{\omega}.
    \label{eq:accepted-inner}
\end{equation}
Each batch starts a fresh search from $V=0$ without carrying perturbations between batches.

\subsection{Discern: Learning from Hard Latent Variants}
\label{sec:train-to-separate}

The \textit{Discern} step trains the auditor on the prefix representations induced by the perturbation $\widehat V_{\mathcal B}$ returned by the \textit{Dream} step as depicted in Figure~\ref{fig:scts-overview}, \ding{174}.
Each prefix retains its recorded trajectory outcome $Y_q$ as the prediction target, since the \textit{Dream} step changes the auditor's representation computation while keeping the input $X_q$ unchanged.
To learn from these representation changes, we hold $\widehat V_{\mathcal B}$ fixed and take an optimization step using only the perturbed batch loss:
\begin{equation}
    \min_{\Theta}\;
    \mathcal L_{\mathcal B}\!\left(
        \Theta,\operatorname{sg}[\widehat V_{\mathcal B}]
    \right),
    \label{eq:outer}
\end{equation}
where $\operatorname{sg}$ stops gradients through the \textit{Dream} search.
Gradients of this loss pass through the perturbed forward pass to update the adapters $a$ and classification head $\psi$, while the pretrained weights $\phi_0$ remain frozen.
If the \textit{Dream} step returns zero, Equation~\ref{eq:outer} uses the unperturbed batch loss $\mathcal L_{\mathcal B}(\Theta)$.

\noindent \textbf{Interpreting the \textit{Discern} loss.}
To explain what hard latent variants add to training on the recorded prefixes, we relate the loss used by \textit{Discern} to the unperturbed prediction loss.
The following result connects their difference to the current auditor's sensitivity to the permitted low-rank operator changes and the perturbation radius returned by \textit{Dream}.
\begin{preprintproposition}
\begin{proposition}[Loss at the returned perturbation]
\label{prop:accepted-challenge}
Fix $\Theta$ and $\mathcal B$, and consider the \textit{Dream} step in Algorithm~\ref{alg:dream2discern} in exact arithmetic.
Assume that, for some $\beta\geq0$, the first-order Taylor remainder of $\mathcal L_{\mathcal B}(\Theta,V)$ at $V=0$ has magnitude at most $\frac{\beta}{2}\lVert V\rVert_\omega^2$ for every $V$ on the segment from $0$ to $\overline V_{\mathcal B}$.
Then
\begin{equation}
    \left|\mathcal L_{\mathcal B}(\Theta,\widehat V_{\mathcal B})
    -\mathcal L_{\mathcal B}(\Theta,0)-\rho_{\mathcal B}g_r\right|
    \leq \frac{\beta}{2}\rho_{\mathcal B}^2.
    \label{eq:first-order-hard-loss}
\end{equation}
\end{proposition}
\end{preprintproposition}
To first order, the loss evaluated by \textit{Discern} is the unperturbed prediction loss plus $\rho_{\mathcal B}g_r$.
By Proposition~\ref{prop:linear-oracle}, $g_r$ is the largest first-order loss increase per unit relative radius under the rank constraints.
The additional loss therefore reflects the current auditor's sensitivity to these operator changes, with $\rho_{\mathcal B}$ accounting for the step actually returned by \textit{Dream}.
The remainder bounds the nonlinear deviation from this local interpretation.
The proof is given in Appendix~\ref{app:proof-accepted-challenge}, and implementation details are in Appendix~\ref{app:solver}. Repeating \textit{Dream} and \textit{Discern} trains $f_\Theta$ to predict outcomes from variants selected using its current loss.
At runtime, the resulting auditor predicts the outcome with $V=0$.

\subsection{Runtime Auditing}
\label{sec:runtime-decision}

At each nonterminal interaction step, the trained auditor assigns a failure score to the observed prefix and predicts the terminal outcome (Figure~\ref{fig:scts-overview}, \ding{175}):
\begin{equation}
    s_{i,t}=f_\Theta(X_{i,t}),\qquad
    \widehat Y_{i,t}=\mathbf 1\{s_{i,t}\ge\zeta_t\},\qquad 1\le t<T_i.
    \label{eq:runtime-decision}
\end{equation}
Each prediction requires one auditor forward pass and a scalar threshold comparison.
To account for variation in score distributions across interaction steps, we fit step-dependent thresholds from successful trajectories~\citep{gu2026safe}.
We first express each calibration score as its empirical percentile at the corresponding step. We then summarize each trajectory by its maximum percentile over nonterminal steps, an aggregation closely related to the max-rank procedure~\citep{timans2023max}.
A high quantile of these maxima defines a shared rank cutoff which is mapped back to the score scale.
The thresholds are fitted offline using successful trajectories from the training partition and remain fixed at runtime. Appendix~\ref{app:threshold} gives the complete construction and calibration protocol.

\section{Experiments}
\label{sec:exp}
\vspace{-1mm}

\subsection{Experimental Setups}
\label{sec:exp-setup}

\noindent \textbf{Benchmarks.}
We evaluate runtime auditing on \textsc{AFTraj-2K}~\citep{zhang2026agentforesight} and StepShield~\citep{felicia2026stepshield}.
\textsc{AFTraj-2K} contains multi-agent executions spanning mathematical reasoning, coding, and tool use, while StepShield provides annotated agent trajectories for safety monitoring.
We train \method{} separately for each benchmark and evaluate it on the corresponding test partition.
The test split of \textsc{AFTraj-2K} comes from the same corpus as the training split and provides in-domain evaluation, whereas the test trajectories of StepShield come from incident templates never used in training and provide out-of-domain evaluation at the template level.


\noindent \textbf{Baselines.}
We compare \method{} with four categories of baselines.
\textbf{(1) Open-weight LLM judges} include Qwen3-4B/8B/32B~\citep{yang2025qwen3}, Llama-3.2-3B-Instruct, Llama-3.1-8B-Instruct~\citep{grattafiori2024llama}, and LlamaGuard3-8B~\citep{inan2023llama}.
\textbf{(2) Closed-source LLM judges} include GPT-5.4, Claude-Sonnet-4.6, Grok-4.3, and Kimi-K2.7-Code.
\textbf{(3) Token uncertainty baselines} use maximum or average token entropy~\citep{malinin2020uncertainty} and negative log-likelihood~\citep{ren2022out} under a Qwen3-8B proxy model to score the recorded prefixes.
\textbf{(4) Failure auditors} include AgentDebug~\citep{zhu2025llm}, AgentForesight~\citep{zhang2026agentforesight}, SAFE~\citep{gu2026safe}, Hide-and-Seek~\citep{park2026hide}, and TRACES~\citep{li2026traces}.

\noindent \textbf{Evaluation metrics.}
We evaluate every nonterminal prefix and aggregate predictions at the trajectory level.
Using each trajectory's maximum prefix score, we report \textbf{AUROC} for outcome discrimination and \textbf{FPR95} for the false-positive rate at 95\% failure-detection recall.
\textbf{Weighted accuracy (wACC)} measures trajectory-level classification accuracy under the observed class proportions.
\textbf{Time-weighted accuracy (TWA)} rewards earlier detection of failed trajectories while accounting for missed failures and false alarms~\citep{romer2026failure}.
For wACC and TWA, we use fitted thresholds at nominal $\alpha=0.05$, except for proprietary judges whose returned labels define the decisions.

\noindent \textbf{Implementation details.}
We instantiate \method{} on Qwen3-4B, Qwen3-8B, Llama-3.2-3B-Instruct, and Llama-3.1-8B-Instruct.
We optimize the LoRA adapters and prediction head with AdamW.
A 100-step supervised warm start is followed by a 4-epoch training budget for \textit{Dream} and \textit{Discern}.
On Qwen3, we report means and standard deviations over three training seeds for \method{}, SAFE, Hide-and-Seek, and TRACES.
More details are provided in Appendix~\ref{app:exp-training}.

We provide detailed experimental setups in Appendix~\ref{app:setup}.

\subsection{Main Experiments}
\label{exp:main}


\begin{table}[t]
\centering
\small
\vspace{-2pt}
\caption{Main comparison on \textbf{\textsc{AFTraj-$2$K}}~\citep{zhang2026agentforesight} and \textbf{StepShield}~\citep{felicia2026stepshield}. The best and second-best results in each column are shown in \textbf{bold} and \underline{underline}.}
\vspace{-6pt}
\label{tab:main_results}
\setlength{\tabcolsep}{3pt}
\renewcommand{\arraystretch}{1.15}
\resizebox{\columnwidth}{!}{%
\begin{tabular}{@{} >{\centering\arraybackslash}p{8mm} l *{8}{c} @{}}
\toprule
& &
\multicolumn{4}{c}{\textbf{\textsc{AFTraj-$2$K}}} &
\multicolumn{4}{c}{\textbf{StepShield}} \\
\cmidrule(lr){3-6}\cmidrule(lr){7-10}

& \textbf{Method} &
\textbf{AUROC}$\uparrow$ & \textbf{FPR95}$\downarrow$ &
\textbf{wACC}$\uparrow$ & \textbf{TWA}$\uparrow$ &
\textbf{AUROC}$\uparrow$ & \textbf{FPR95}$\downarrow$ &
\textbf{wACC}$\uparrow$ & \textbf{TWA}$\uparrow$ \\
\midrule

\multirow{6}{*}{\rotatebox{90}{\textbf{Open.}}} & 
Qwen3-4B &
$69.63$ & $86.39$ & $55.12$ & $51.65$ &
$81.07$ & $82.41$ & $50.46$ & $49.12$ \\
& Qwen3-8B &
$67.18$ & $80.47$ & $56.93$ & $52.83$ &
$84.95$ & $88.89$ & $62.96$ & $53.08$ \\
& Qwen3-32B & $69.79$
& $82.25$ & $56.33$ & $52.70$ &
$90.70$ & $70.37$ & $58.80$ & $52.31$ \\
& Llama3.2-3B-it &
$59.60$ & $90.53$ & $54.22$ & $51.21$ &
$78.88$ & $86.11$ & $50.93$ & $50.28$ \\
& Llama3.1-8B-it &
$58.02$ & $94.08$ & $53.01$ & $51.49$ &
$83.49$ & $76.85$ & $50.00$ & $49.56$ \\
& LlamaGuard3-8B &
$57.45$ & $94.08$ & $54.22$ & $50.07$ &
$83.02$ & $86.11$ & $75.00$ & $56.05$ \\
\midrule

\multirow{4}{*}{\rotatebox{90}{\textbf{Closed.}}} &
GPT-5.4 &
$50.53$ & $99.41$ & $45.48$ & $37.53$ &
$88.30$ & $58.33$ & $65.74$ & $41.90$ \\
& Claude-Sonnet-4.6 &
$70.90$ & $100.00$ & $63.86$ & $54.83$ &
$\mathbf{95.58}$ & $\mathbf{38.89}$ & $87.04$ & $\underline{56.79}$ \\
& Grok-4.3 &
$71.57$ & $95.27$ & $65.96$ & $52.34$ &
$91.48$ & $58.33$ & $82.87$ & $54.31$ \\
& Kimi-K2.7-Code &
$73.57$ & $100.00$ & $68.67$ & $54.43$ &
$90.50$ & $54.07$ & $85.19$ & $54.38$ \\
\midrule

\multirow{4}{*}{\rotatebox{90}{\textbf{Token Unc.}}} &
Max Entropy-8B &
$53.58$	& $95.86$	& $49.10$	& $47.92$ &
$53.70$	& $90.74$	& $49.54$	& $48.87$ \\
& Avg. Entropy-8B &
$66.20$	& $90.53$	& $49.70$	& $47.26$ &
$57.54$	& $91.67$	& $50.46$	& $49.46$ \\
& Max NLL-8B &
$45.20$	& $97.63$	& $51.20$	& $48.27$ &
$51.58$	& $93.52$	& $50.00$	& $49.02$ \\
& Avg. NLL-8B &
$48.09$	& $99.41$	& $54.52$	& $50.48$ &
$57.58$	& $87.96$	& $50.46$	& $50.20$ \\
\midrule

& AgentDebug-8B &
$56.78$	& $87.57$	& $51.81$	& $48.73$ &
$75.77$	& $85.19$	& $58.33$	& $51.38$ \\
& AgentForesight-7B &
$83.75$ & $89.94$ & $77.11$ & $59.33$ &
$83.37$ & $77.78$ & $79.17$ & $56.13$ \\
& \cellcolor{baselinegray}SAFE-Qwen3-4B &
\cellcolor{baselinegray}$94.76^{\pm 0.39}$ & \cellcolor{baselinegray}$41.42^{\pm 3.13}$ &
\cellcolor{baselinegray}$84.74^{\pm 1.94}$ & \cellcolor{baselinegray}$68.83^{\pm 0.81}$ &
\cellcolor{baselinegray}$84.47^{\pm 0.72}$ & \cellcolor{baselinegray}$63.27^{\pm 6.17}$ &
\cellcolor{baselinegray}$66.20^{\pm 5.09}$ & \cellcolor{baselinegray}$52.40^{\pm 1.71}$ \\
& \cellcolor{baselinegray}SAFE-Qwen3-8B &
\cellcolor{baselinegray}$92.69^{\pm 0.81}$ & \cellcolor{baselinegray}$55.62^{\pm 11.34}$ &
\cellcolor{baselinegray}$85.74^{\pm 1.22}$ & \cellcolor{baselinegray}$66.57^{\pm 0.39}$ &
\cellcolor{baselinegray}$79.45^{\pm 3.01}$ & \cellcolor{baselinegray}$63.89^{\pm 10.43}$ &
\cellcolor{baselinegray}$61.73^{\pm 6.96}$ & \cellcolor{baselinegray}$50.67^{\pm 1.55}$ \\
& \cellcolor{baselinegray}H\&S-Qwen3-4B &
\cellcolor{baselinegray}$90.84^{\pm 1.26}$ & \cellcolor{baselinegray}$72.98^{\pm 11.96}$ &
\cellcolor{baselinegray}$88.76^{\pm 1.14}$ & \cellcolor{baselinegray}$70.27^{\pm 0.43}$ &
\cellcolor{baselinegray}$85.02^{\pm 0.96}$ & \cellcolor{baselinegray}$86.42^{\pm 5.66}$ &
\cellcolor{baselinegray}$70.99^{\pm 2.83}$ & \cellcolor{baselinegray}$54.74^{\pm 1.33}$ \\
& \cellcolor{baselinegray}H\&S-Qwen3-8B &
\cellcolor{baselinegray}$89.54^{\pm 0.57}$ & \cellcolor{baselinegray}$89.15^{\pm 6.73}$ &
\cellcolor{baselinegray}$87.45^{\pm 1.42}$ & \cellcolor{baselinegray}$69.01^{\pm 1.42}$ &
\cellcolor{baselinegray}$81.58^{\pm 2.56}$ & \cellcolor{baselinegray}$81.17^{\pm 7.54}$ &
\cellcolor{baselinegray}$69.75^{\pm 0.96}$ & \cellcolor{baselinegray}$53.82^{\pm 0.69}$ \\
& \cellcolor{baselinegray}TRACES-Qwen3-4B &
\cellcolor{baselinegray}$91.98^{\pm 1.18}$ & \cellcolor{baselinegray}$60.75^{\pm 16.25}$ &
\cellcolor{baselinegray}$88.35^{\pm 1.14}$ & \cellcolor{baselinegray}$69.90^{\pm 1.38}$ &
\cellcolor{baselinegray}$85.37^{\pm 2.72}$ & \cellcolor{baselinegray}$59.88^{\pm 7.98}$ &
\cellcolor{baselinegray}$64.66^{\pm 2.28}$ & \cellcolor{baselinegray}$52.45^{\pm 2.52}$ \\
& \cellcolor{baselinegray}TRACES-Qwen3-8B &
\cellcolor{baselinegray}$92.63^{\pm 0.64}$ & \cellcolor{baselinegray}$65.88^{\pm 12.32}$ &
\cellcolor{baselinegray}$87.05^{\pm 1.04}$ & \cellcolor{baselinegray}$69.91^{\pm 0.48}$ &
\cellcolor{baselinegray}$79.25^{\pm 1.86}$ & \cellcolor{baselinegray}$79.63^{\pm 12.04}$ &
\cellcolor{baselinegray}$70.83^{\pm 2.12}$ & \cellcolor{baselinegray}$48.52^{\pm 2.66}$ \\
& \cellcolor{sctshighlight}\textbf{Dream2Discern-Qwen3-4B} &
\cellcolor{sctshighlight}$\underline{95.45}^{\pm 0.16}$ & \cellcolor{sctshighlight}$\underline{38.86}^{\pm 7.81}$ &
\cellcolor{sctshighlight}$\underline{90.96}^{\pm 0.60}$ & \cellcolor{sctshighlight}$\underline{72.61}^{\pm 0.38}$ &
\cellcolor{sctshighlight}$93.62^{\pm 0.68}$ & \cellcolor{sctshighlight}$\underline{42.28}^{\pm 2.33}$ &
\cellcolor{sctshighlight}$\mathbf{90.28}^{\pm 0.46}$ & \cellcolor{sctshighlight}$\mathbf{57.67}^{\pm 0.35}$ \\
\multirow{-10}{*}{\rotatebox[origin=c]{90}{\textbf{Failure Auditors}}} &
\cellcolor{sctshighlight}\textbf{Dream2Discern-Qwen3-8B} &
\cellcolor{sctshighlight}$\mathbf{95.82}^{\pm 0.43}$ & \cellcolor{sctshighlight}$\mathbf{31.95}^{\pm 3.07}$ &
\cellcolor{sctshighlight}$\mathbf{91.47}^{\pm 0.17}$ & \cellcolor{sctshighlight}$\mathbf{73.45}^{\pm 0.71}$ &
\cellcolor{sctshighlight}$\underline{93.95}^{\pm 0.63}$ & \cellcolor{sctshighlight}$53.09^{\pm 8.40}$ &
\cellcolor{sctshighlight}$\underline{87.50}^{\pm 0.93}$ & \cellcolor{sctshighlight}$55.69^{\pm 1.10}$ \\
\bottomrule
\end{tabular}}
\vspace{-6pt}
\end{table}

\noindent \textbf{Comparison with baseline methods.}
Table~\ref{tab:main_results} shows that \method{} leads all four metrics on \textsc{AFTraj-2K} and wACC and TWA on StepShield.
LLM judges are more competitive on StepShield, while token uncertainty-based methods provide limited outcome discrimination on both benchmarks.
For both backbones, \method{} outperforms competitive failure auditor baselines SAFE, Hide-and-Seek, and TRACES on all four metrics averaged over three seeds.
Its wACC advantage is substantially larger on out-of-domain StepShield than on in-domain \textsc{AFTraj-2K}, supporting our strategy of learning from hard latent variants without collecting additional trajectories.
Our 4B auditor also exceeds Claude-Sonnet-4.6 in StepShield wACC and TWA, demonstrating effective runtime auditing with a compact backbone.
Together, these results demonstrate the effectiveness of \method{}, which combines loss-guided variant construction with outcome-supervised representation learning for runtime failure prediction.

\noindent \textbf{Generalization across model families.}
To assess whether our training strategy extends beyond the Qwen3 family, we further compare \method{} with SAFE, H\&S, and TRACES using the same Llama3 backbones.
Table~\ref{tab:cross-family} shows that \method{} achieves the highest wACC on both benchmarks at both model sizes and leads all three reported metrics on \textsc{AFTraj-2K}.
These results support the applicability of our training strategy across model families for runtime failure prediction.

\subsection{Ablation and Further Analysis}
\label{exp:abla}

\noindent \textbf{Ablation on the perturbation strategy.}
To assess the contribution of \textit{Dream}, we compare \method{} with the same auditor trained without perturbations or with random low-rank perturbations.
The unperturbed variant trains the same LoRA adapters and prediction head using the prediction loss.
The random variant retains the same perturbed operators, rank cap, and nominal radius, but replaces the loss-guided search with random sampling.
As shown in Table~\ref{tab:training-ablation-review}, training with \textit{Dream} consistently achieves the highest AUROC in all four settings.
It also improves wACC and TWA over random perturbation training in most settings.
These results support using prediction loss to construct hard latent variants that improve the auditor's discrimination between successful and failed executions.
Section~\ref{exp:mecha} further examines how the selected perturbations affect the auditor.

\noindent \textbf{Detection accuracy--timeliness trade-off.}
We next examine whether \method{} retains its predictive gains at operating points that favor earlier failure warnings, as a reliable runtime auditor should balance accuracy and promptness.
Based on Qwen3-8B, we vary the nominal false-alarm level $\alpha$ from $0.01$ to $1.00$ to obtain various decision thresholds for \method{}, SAFE, H\&S, and TRACES.
Figure~\ref{fig:further-analysis} plots wACC against mean normalized detection time over all failed trajectories, with undetected failures assigned a value of one.
On \textsc{AFTraj-2K}, \method{} achieves higher wACC when early detection is prioritized, showing that its predictive gains are compatible with timely failure warnings.
On StepShield, these gains extend across a broad range of operating points, supporting the effectiveness of \method{} under different accuracy-timeliness requirements.

\begin{table}[t]
\begin{minipage}[t]{0.485\textwidth}
\centering\small
\caption{Comparison on Llama backbones.}
\vspace{-5pt}
\label{tab:cross-family}
\setlength{\tabcolsep}{5pt}
\renewcommand{\arraystretch}{1.15}
\resizebox{\linewidth}{!}{%
\begin{tabular}{@{}lllccc@{}}
\toprule
\textbf{Benchmark} & \textbf{Base Model} & \textbf{Method} &
\textbf{AUROC}$\uparrow$ & \textbf{wACC}$\uparrow$ & \textbf{TWA}$\uparrow$ \\
\midrule

\multirow{8}{*}{\textbf{\textsc{AFTraj-2K}}}
  & \multirow{4}{*}{Llama3.2-3B-it}
    & SAFE   & $\underline{94.62}^{\pm 0.09}$ & $\underline{85.54}^{\pm 4.64}$ & $\underline{69.12}^{\pm 2.94}$ \\
  & & H\&S   & $91.46^{\pm 0.67}$ & $83.23^{\pm 4.52}$ & $61.98^{\pm 5.25}$ \\
  & & TRACES & $91.88^{\pm 0.57}$ & $\underline{85.54}^{\pm 4.44}$ & $66.73^{\pm 4.30}$ \\
  & & \cellcolor{sctshighlight}\textbf{\method{}}
    & \cellcolor{sctshighlight}$\mathbf{95.62}^{\pm 0.61}$
    & \cellcolor{sctshighlight}$\mathbf{91.27}^{\pm 0.30}$
    & \cellcolor{sctshighlight}$\mathbf{73.38}^{\pm 0.18}$ \\
\cmidrule(l){2-6}
  & \multirow{4}{*}{Llama3.1-8B-it}
    & SAFE   & $\underline{94.64}^{\pm 0.37}$ & $86.75^{\pm 2.17}$ & $68.89^{\pm 1.65}$ \\
  & & H\&S   & $88.55^{\pm 0.93}$ & $71.29^{\pm 7.58}$ & $56.96^{\pm 4.66}$ \\
  & & TRACES & $92.64^{\pm 1.30}$ & $\underline{88.15}^{\pm 2.26}$ & $\underline{69.75}^{\pm 2.16}$ \\
  & & \cellcolor{sctshighlight}\textbf{\method{}}
    & \cellcolor{sctshighlight}$\mathbf{95.91}^{\pm 0.35}$
    & \cellcolor{sctshighlight}$\mathbf{91.16}^{\pm 0.63}$
    & \cellcolor{sctshighlight}$\mathbf{72.30}^{\pm 0.54}$ \\
\midrule

\multirow{8}{*}{\textbf{StepShield}}
  & \multirow{4}{*}{Llama3.2-3B-it}
    & SAFE   & $66.18^{\pm 0.83}$ & $59.57^{\pm 0.53}$ & $51.28^{\pm 1.33}$ \\
  & & H\&S   & $71.21^{\pm 7.03}$ & $58.18^{\pm 2.71}$ & $\underline{51.57}^{\pm 1.26}$ \\
  & & TRACES & $\underline{76.56}^{\pm 3.20}$ & $\underline{62.04}^{\pm 2.45}$ & $50.44^{\pm 4.21}$ \\
  & & \cellcolor{sctshighlight}\textbf{\method{}}
    & \cellcolor{sctshighlight}$\mathbf{90.08}^{\pm 1.78}$
    & \cellcolor{sctshighlight}$\mathbf{82.25}^{\pm 4.20}$
    & \cellcolor{sctshighlight}$\mathbf{55.01}^{\pm 2.26}$ \\
\cmidrule(l){2-6}
  & \multirow{4}{*}{Llama3.1-8B-it}
    & SAFE   & $82.82^{\pm 2.41}$ & $64.20^{\pm 0.71}$ & $\underline{53.61}^{\pm 0.24}$ \\
  & & H\&S   & $\underline{87.54}^{\pm 1.81}$ & $59.26^{\pm 1.22}$ & $52.61^{\pm 1.33}$ \\
  & & TRACES & $\mathbf{89.59}^{\pm 2.65}$ & $\underline{77.78}^{\pm 3.24}$ & $\mathbf{56.01}^{\pm 0.72}$ \\
  & & \cellcolor{sctshighlight}\textbf{\method{}}
    & \cellcolor{sctshighlight}$86.78^{\pm 1.92}$
    & \cellcolor{sctshighlight}$\mathbf{79.48}^{\pm 2.94}$
    & \cellcolor{sctshighlight}$49.06^{\pm 2.67}$ \\
\bottomrule
\end{tabular}}
\end{minipage}
\hfill
\begin{minipage}[t]{0.485\textwidth}
\centering\small
\caption{Ablation on the perturbation strategy.}
\vspace{-5pt}
\label{tab:training-ablation-review}
\setlength{\tabcolsep}{4pt}
\renewcommand{\arraystretch}{1.238}
\resizebox{\linewidth}{!}{%
\begin{tabular}{@{}lllccc@{}}
\toprule
\textbf{Benchmark} & \textbf{Base Model} & \textbf{Pert. Strategy} &
\textbf{AUROC}$\uparrow$ & \textbf{wACC}$\uparrow$ & \textbf{TWA}$\uparrow$ \\
\midrule

\multirow{6}{*}{\textbf{\textsc{AFTraj-2K}}}
  & \multirow{3}{*}{Qwen3-4B}
    & w/o Perturbation & $\underline{95.01}$ & $\mathbf{90.96}$ & $\underline{71.75}$ \\
  & & w/ Random Pert.  & $94.40$ & $89.16$ & $68.24$ \\
  & & \cellcolor{sctshighlight}\textbf{w/ \emph{Dream} Pert.}
    & \cellcolor{sctshighlight}$\mathbf{95.31}$
    & \cellcolor{sctshighlight}$\underline{90.36}$
    & \cellcolor{sctshighlight}$\mathbf{72.28}$ \\
\cmidrule(l){2-6}
  & \multirow{3}{*}{Qwen3-8B}
    & w/o Perturbation & $93.65$ & $\underline{90.06}$ & $\underline{71.04}$ \\
  & & w/ Random Pert.  & $\underline{94.60}$ & $89.46$ & $69.80$ \\
  & & \cellcolor{sctshighlight}\textbf{w/ \emph{Dream} Pert.}
    & \cellcolor{sctshighlight}$\mathbf{96.21}$
    & \cellcolor{sctshighlight}$\mathbf{91.57}$
    & \cellcolor{sctshighlight}$\mathbf{73.70}$ \\
\midrule

\multirow{6}{*}{\textbf{StepShield}}
  & \multirow{3}{*}{Qwen3-4B}
    & w/o Perturbation & $89.78$ & $\underline{84.72}$ & $\underline{53.10}$ \\
  & & w/ Random Pert.  & $\underline{91.20}$ & $84.26$ & $52.78$ \\
  & & \cellcolor{sctshighlight}\textbf{w/ \emph{Dream} Pert.}
    & \cellcolor{sctshighlight}$\mathbf{92.91}$
    & \cellcolor{sctshighlight}$\mathbf{90.28}$
    & \cellcolor{sctshighlight}$\mathbf{57.48}$ \\
\cmidrule(l){2-6}
  & \multirow{3}{*}{Qwen3-8B}
    & w/o Perturbation & $93.82$ & $87.96$ & $\underline{57.32}$ \\
  & & w/ Random Pert.  & $\underline{94.37}$ & $\mathbf{89.81}$ & $\mathbf{57.72}$ \\
  & & \cellcolor{sctshighlight}\textbf{w/ \emph{Dream} Pert.}
    & \cellcolor{sctshighlight}$\mathbf{94.59}$
    & \cellcolor{sctshighlight}$\underline{88.43}$
    & \cellcolor{sctshighlight}$56.42$ \\
\bottomrule
\end{tabular}}
\end{minipage}
\vspace{-8pt}
\end{table}

\subsection{Mechanism Analysis}
\label{exp:mecha}

We analyze how \textit{Dream} constructs hard latent variants and how sensitivity to its perturbations relates to held-out prediction errors.
Appendix~\ref{app:mechanism-analysis} provides detailed experimental settings and analyses.

\noindent \textbf{Variant construction.}
We first examine whether \textit{Dream} constructs harder latent variants than random perturbations do.
After training Qwen3-8B with \method{}, we fix its learned parameters and sample 50 batches of 32 prefixes from the \textsc{AFTraj-2K} training set.
For each batch, we construct a loss-guided \textit{Dream} perturbation and sample random perturbations using the same settings.
We temporarily apply each candidate and measure the change in batch BCE relative to the unperturbed auditor on the same prefixes.
Each point in Figure~\ref{fig:loss-guided-vs-random} summarizes one batch, using the median of two draws for random perturbations.
\textit{Dream} increases loss on every batch, whereas random changes cluster near zero.
This contrast shows that loss-guided search finds representation changes that make outcome prediction harder for the fixed auditor.

\begin{figure}[t]
    \begin{minipage}[t]{0.485\textwidth}
    \centering
        \begin{subfigure}[b]{0.49\textwidth}
            \includegraphics[width=\linewidth]{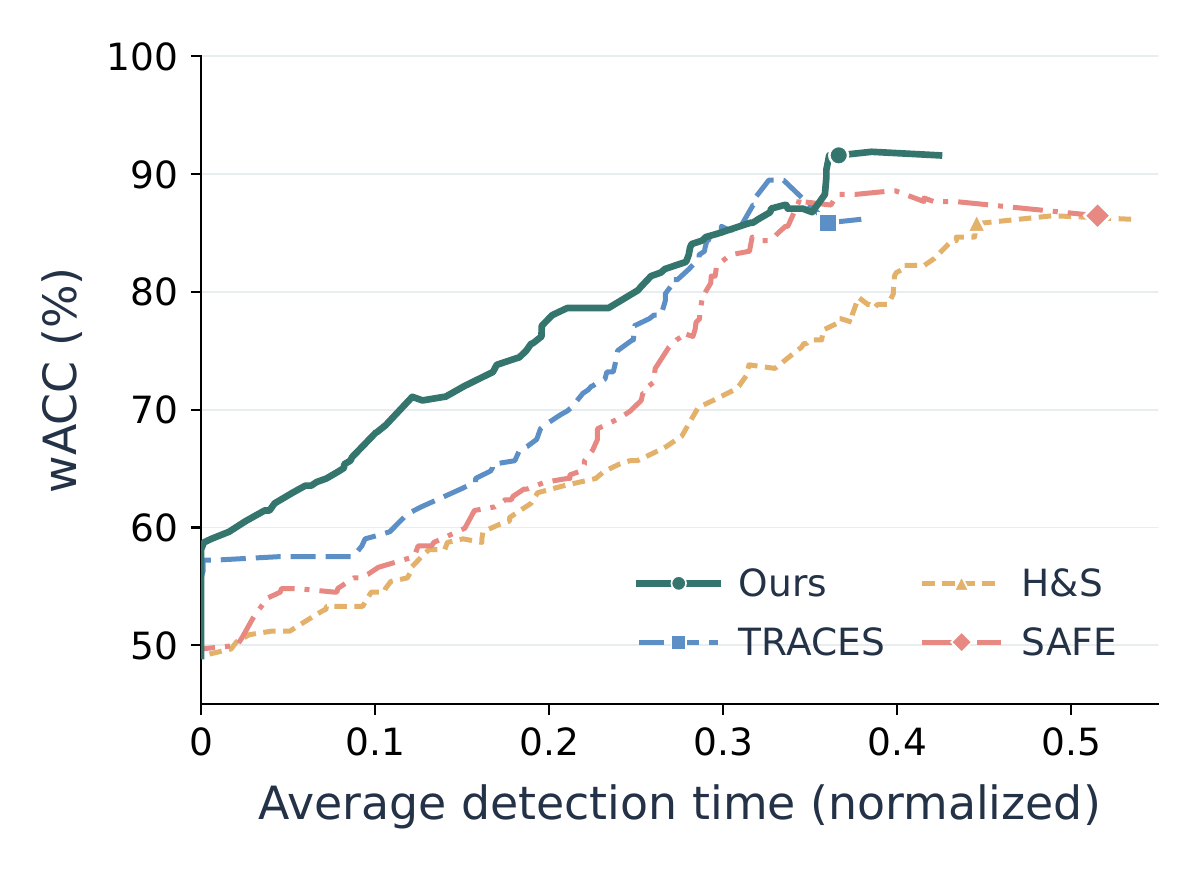}
            \caption{\small{\textsc{AFTraj-2K}}}
            \label{fig:tradeoff-aftraj}
        \end{subfigure}
        \hfill
        \begin{subfigure}[b]{0.49\textwidth}
            \includegraphics[width=\linewidth]{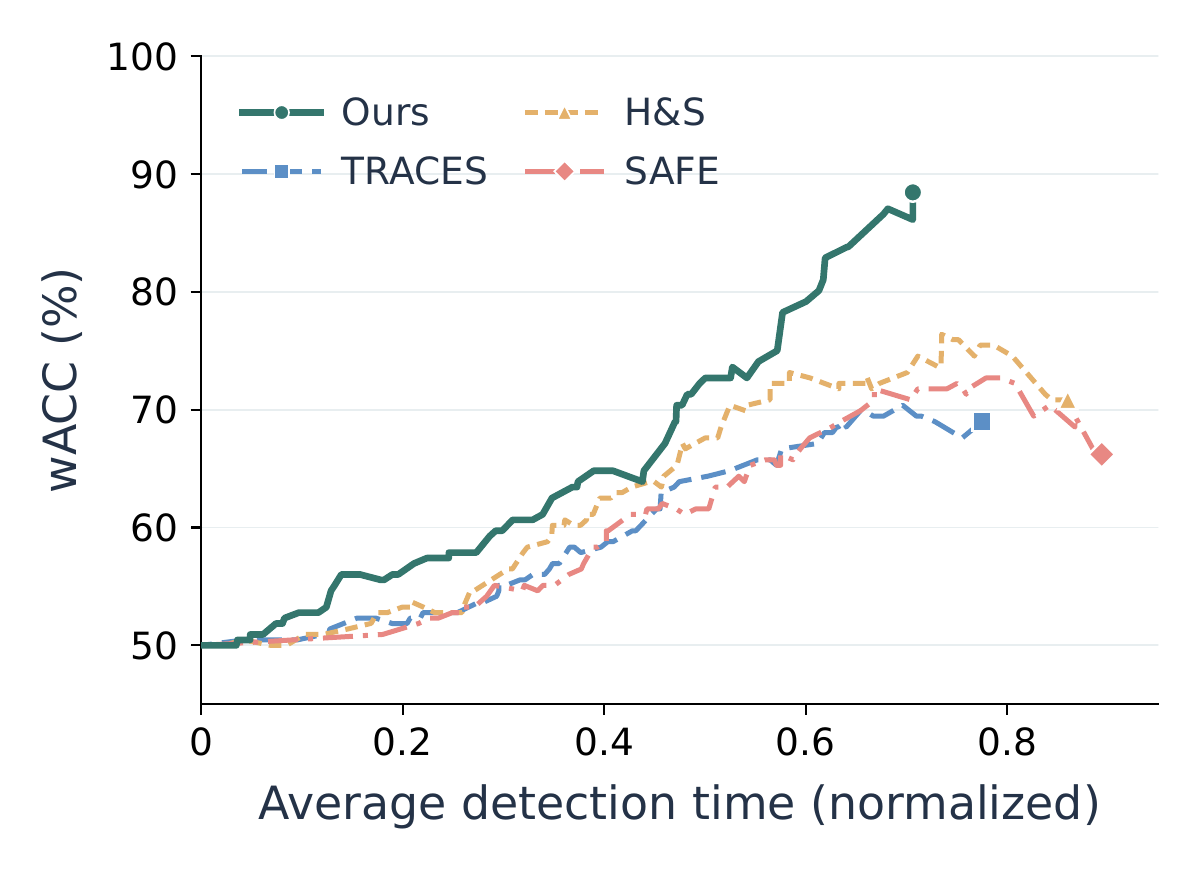}
            \caption{\small{StepShield}}
            \label{fig:tradeoff-stepshield}
        \end{subfigure}
    \caption{Detection accuracy--timeliness trade-offs for Qwen3-8B. Markers denote $\alpha=0.05$.}
    \label{fig:further-analysis}
    \end{minipage}
    \hfill
    \begin{minipage}[t]{0.485\textwidth}
        \begin{subfigure}[b]{0.49\textwidth}
            \includegraphics[width=\linewidth]{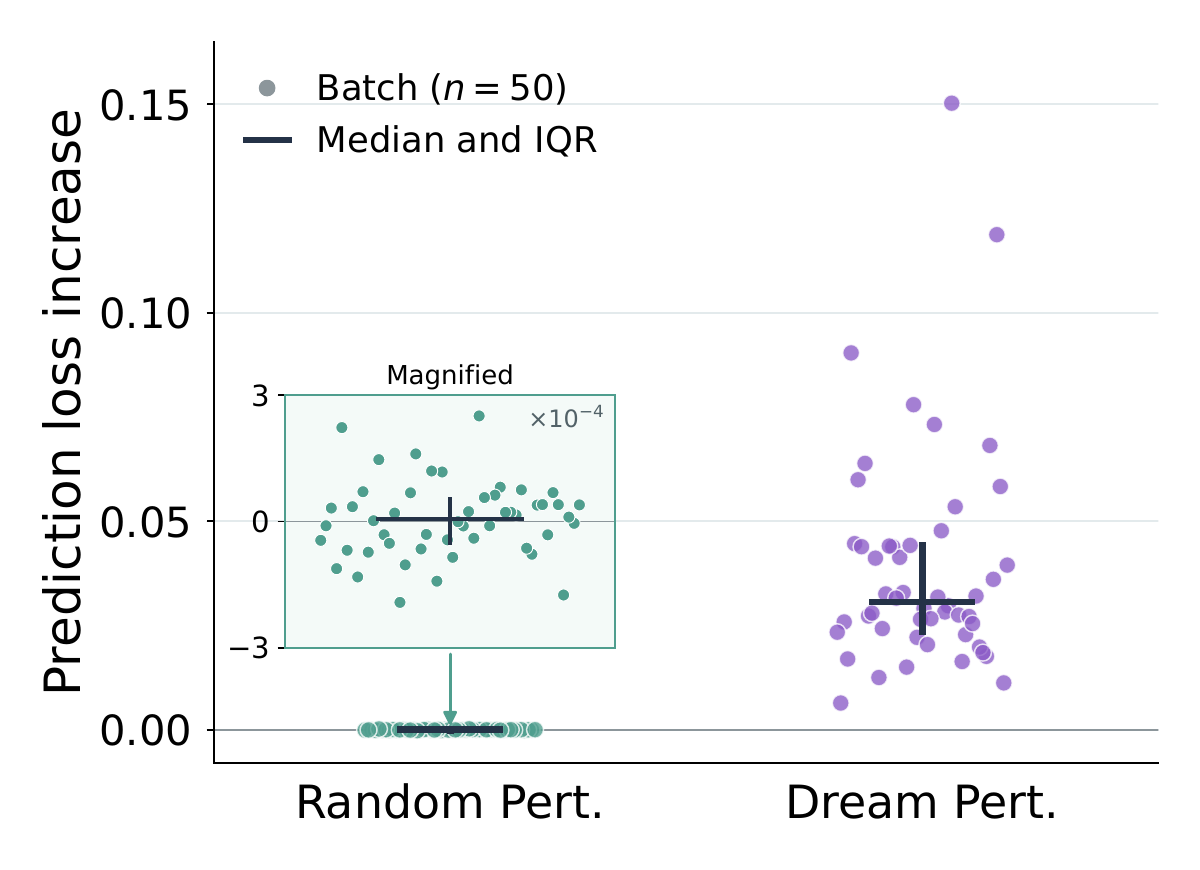}
            \caption{Per-batch loss increase}
            \label{fig:loss-guided-vs-random}
        \end{subfigure}
        \hfill
        \begin{subfigure}[b]{0.49\textwidth}
            \includegraphics[width=\linewidth]{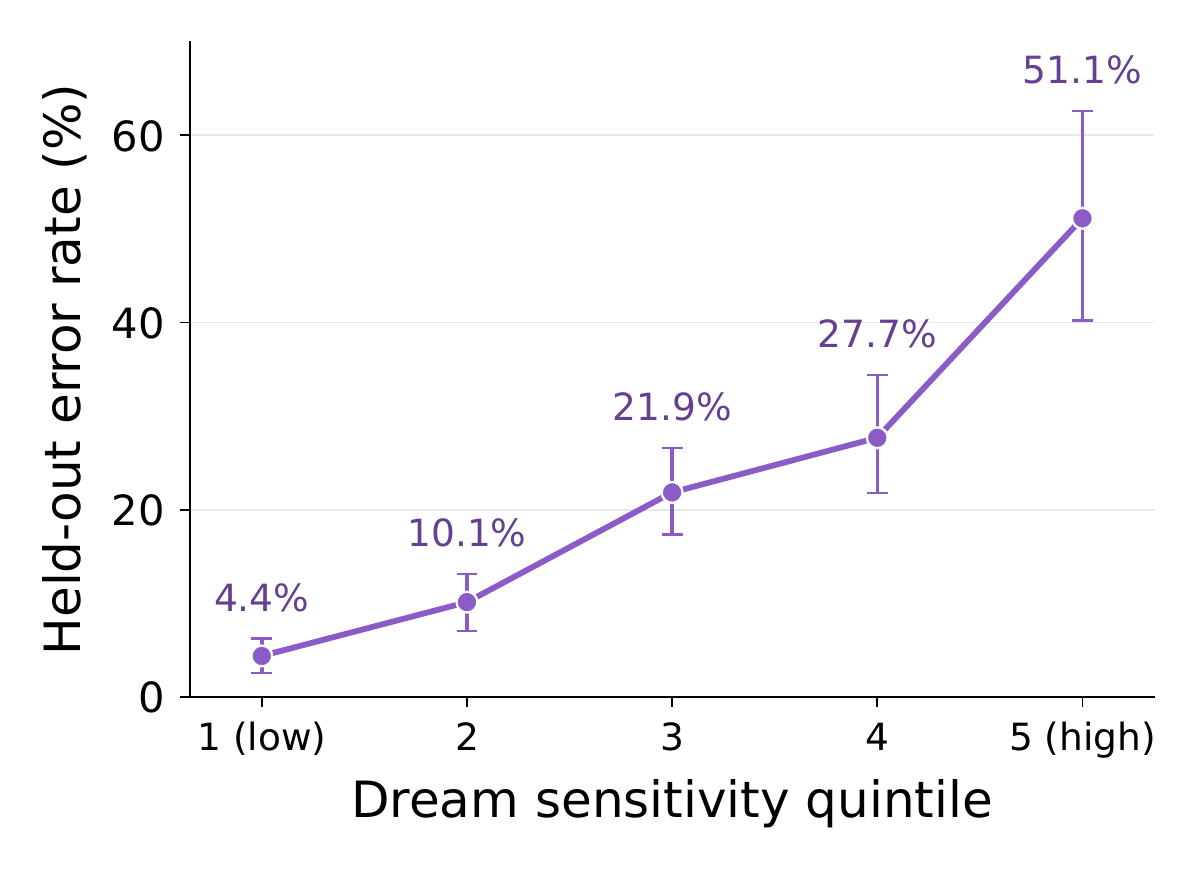}
            \caption{Dream sensitivity}
            \label{fig:dream_sensitivity}
        \end{subfigure}
    \caption{\textbf{(a)} Per-batch loss increases  and \textbf{(b)} Dream sensitivity - Held-out error association.}
    \label{fig:mechanism-analysis}
    \end{minipage}
\end{figure}

\noindent \textbf{Dream sensitivity.}
We next examine whether perturbations selected on training batches are informative about the auditor's prediction errors on held-out prefixes.
We use a fixed Qwen3-4B auditor trained without perturbations and construct eight \textit{Dream} perturbations from separate training batches.
For each held-out \textsc{AFTraj-2K} prefix, we measure the absolute change in its failure score under each perturbation and average these changes to obtain its \emph{Dream sensitivity}.
We sort prefixes by sensitivity into five equal-sized groups and compute each group's error rate from the unperturbed predictions and recorded outcomes.
Figure~\ref{fig:dream_sensitivity} shows that error rates increase from the lowest to the highest sensitivity group.
This association suggests that loss-guided perturbations reflect weaknesses in the auditor's learned outcome discrimination, supporting the construction of hard latent variants around prediction difficulties that also arise on held-out prefixes.

\subsection{Case Study}
\label{exp:case}

\begin{wrapfigure}{r}{0.59\linewidth}
\centering
\vspace{-20pt}
\includegraphics[width=\linewidth]{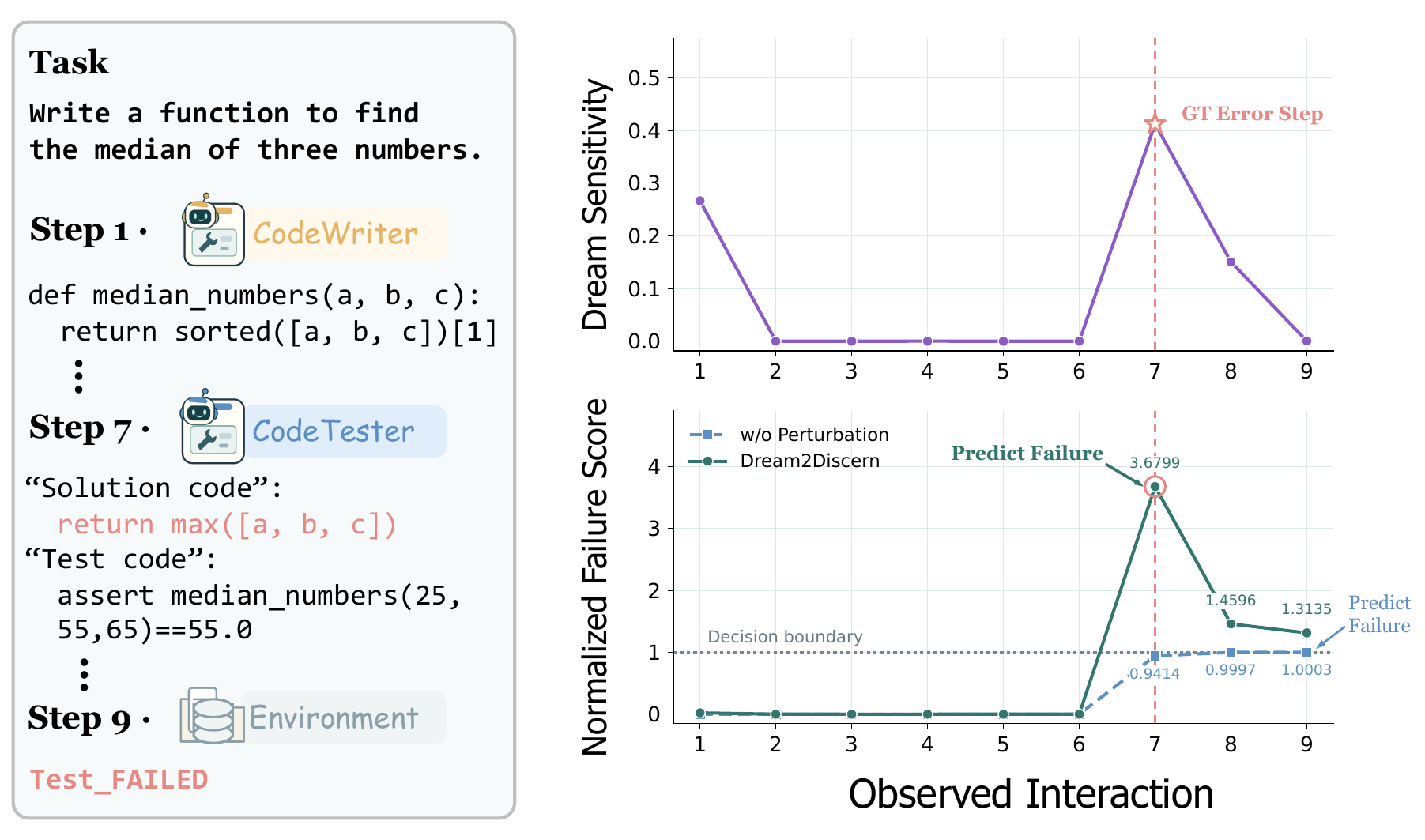}
\vspace{-12pt}
\caption{Case study of runtime failure prediction on a held-out coding trajectory from \textsc{AFTraj-2K}.}
\label{fig:case}
\vspace{-10pt}
\end{wrapfigure}

\leavevmode\noindent \textbf{Qualitative case studies.} Figure~\ref{fig:case} illustrates a held-out trajectory where \method{} predicts the eventual failure at the ground-truth error step, two interactions before the environment reports \texttt{TEST\_FAILED}.
In this case, \texttt{CodeTester} mistakenly replaces \texttt{CodeWriter}'s correct \texttt{median} implementation with a \texttt{max} function during verification, thus leading the task to fail.
As shown in the lower right panel, the auditor trained without perturbations first predicts failure at step~9, where the environment has already reported the test failure.
However, when we temporarily apply \textit{Dream} perturbations to the same auditor, its sensitivity exhibits a pronounced peak at step~7, as shown in the upper-right panel.
This peak shows that the \textit{Dream} perturbations elicit their strongest response at the held-out error step, which the unperturbed auditor has not yet flagged.
The lower right panel further shows that after training on hard latent variants,
\method{} promptly predicts the eventual failure from the observed prefix at the
ground-truth error step.

\section{Related Work}
\label{app:related-work}

\noindent \textbf{LLM-based agentic systems.}
LLM-based agents combine reasoning, actions, and tool use to interact with external environments~\citep{yao2022react,schick2023toolformer,zhang2025cot}.
Frameworks such as AutoGen~\citep{wu2024autogen}, MetaGPT~\citep{hong2023metagpt}, and CAMEL~\citep{li2023camel} organize these capabilities through specialized roles and inter-agent communication.
GPTSwarm~\citep{zhuge2024language} further optimizes prompts and agent connectivity, while Magentic-One~\citep{fourney2024magentic} coordinates web browsing, file navigation, and code execution.
Other work studies how agents retain information across interactions~\citep{sun2026rethinking} and how their reliability changes over a deployed lifespan~\citep{zhu2026your}.
Alongside these advances, reliable deployment requires assessing whether an ongoing execution is likely to fail, motivating our focus on runtime prediction from observed prefixes.

\noindent \textbf{Failure analysis for LLM agents.}
Failure analysis characterizes errors in agent executions and provides guidance for diagnosis and improvement.
Common interface corruptions can also disrupt computer-use agents without adversarial intent~\citep{pmlr-v306-sun26ba}.
MAST~\citep{cemri2026multi} organizes multi-agent failures around system design, inter-agent misalignment, and task verification.
Who\&When~\citep{zhang2025agent} benchmarks attribution to responsible agents and decisive error steps, while AgenTracer~\citep{zhang2025agentracer} uses counterfactual replay and fault injection to construct supervision for a trained failure tracer.
AgentDebug~\citep{zhu2025llm} diagnoses errors across agent components and generates corrective feedback to support recovery.
OAT~\citep{yeh2026tracing} learns trajectory dynamics from successful executions and uses deviations from these dynamics to locate error steps in failed trajectories.
Extending attribution beyond failed executions, \citet{qian2026behind} study which historical interactions and evidence drive an agent's actions.
These studies support diagnosis and improvement, while runtime prediction complements them by assessing terminal outcomes before execution finishes, when intervention remains possible.

\noindent \textbf{Runtime failure prediction.}
Runtime auditors estimate outcomes from partial execution histories to support timely decisions.
For vision-language-action policies, SAFE~\citep{gu2026safe} predicts failure from internal features, while Hide-and-Seek~\citep{park2026hide} discovers temporally structured failure signals through contrastive learning with trajectory-level labels.
For LLM agents, AgentForesight~\citep{zhang2026agentforesight} trains an online auditor with decisive-error annotations, whereas TRACES~\citep{li2026traces} models the temporal evolution of latent risk features from trajectory-level supervision.
PrefixGuard~\citep{huang2026prefixguard} learns prefix-risk monitors from structured trace representations, and \citet{baidya2026evidence} combine sparse failure evidence with a learned stopping policy.
Related safety guardrails include AgentDoG~\citep{liu2026agentdog}, which provides contextual risk diagnosis for agent trajectories.
StepGuard~\citep{zheng2026stepguard} combines context-matched safe and unsafe trajectories with reinforcement learning that balances performance across these classes to assess tool actions before execution.
We focus on the complementary question of how to target the auditor's current prediction difficulties without expanding its trajectory corpus.
\method{} uses its prediction loss to construct hard latent variants of recorded prefixes and learns from them with the original outcome labels, without additional trajectory collection or verification.

\section{Conclusion}

In this paper, we highlight that broader failure coverage alone does not identify the current auditor's hard cases, motivating implicit augmentation guided by prediction loss.
To realize this idea, we present \method{}, a runtime auditor that learns from hard latent variants without acquiring additional labeled trajectories.
\textit{Dream} constructs these variants through loss-guided operator perturbations, and \textit{Discern} learns from them using the original outcome labels.
The learned auditor operates without perturbations at runtime.
Evaluations on \textsc{AFTraj-2K} and StepShield across model families demonstrate improved runtime failure prediction.
These findings support targeting the auditor's current prediction difficulties to improve learning from a fixed trajectory corpus.

\clearpage

\section*{AI Use Statement}

The authors led the research, method development, experiments, and manuscript preparation.
Generative AI tools were used in a supporting role to refine methodological and experimental choices, assist with mathematical formulation and proof development, support implementation and data analysis, and draft or edit portions of the manuscript.
LLM use for trajectory editing, outcome checks, and baseline judgments is documented in Appendices~\ref{app:trajectory-editing} and~\ref{app:exp-baselines}.
The authors critically reviewed and revised AI-assisted material against the implementation and recorded results, made the final scientific and editorial decisions, and take responsibility for the entire submission, including AI-assisted text, claims, code, and artifacts.

\section*{Ethics Statement}

This work studies runtime failure prediction on existing agent-trajectory benchmarks to support reliable agent deployment.
Auditor predictions can be incorrect and do not certify safe execution or prevent harmful actions.
Deployment should retain task-specific safeguards and human oversight, and protect sensitive information in the execution histories supplied to the auditor.

\section*{Reproducibility Statement}

The problem setting and method are specified in Sections~\ref{sec:problem} and~\ref{sec:method}, with the training algorithm and runtime calibration detailed in Appendix~\ref{app:method-details}.
Appendix~\ref{app:proofs} provides the assumptions and complete proofs of the theoretical results.
Dataset partitions, prefix construction, training configurations, baseline implementations, and evaluation metrics are described in Appendix~\ref{app:setup}, and supplementary analysis protocols are given in Appendix~\ref{app:additional-analyses}.

\bibliographystyle{plainnat}
\bibliography{references}

@article{fourney2024magentic,
  title={Magentic-one: A generalist multi-agent system for solving complex tasks},
  author={Fourney, Adam and Bansal, Gagan and Mozannar, Hussein and Tan, Cheng and Salinas, Eduardo and Niedtner, Friederike and Proebsting, Grace and Bassman, Griffin and Gerrits, Jack and Alber, Jacob and others},
  journal={arXiv preprint arXiv:2411.04468},
  year={2024}
}

@article{cemri2026multi,
  title={Why do multi-agent llm systems fail?},
  author={Cemri, Mert and Pan, Melissa Z and Yang, Shuyi and Agrawal, Lakshya A and Chopra, Bhavya and Tiwari, Rishabh and Keutzer, Kurt and Parameswaran, Aditya and Klein, Dan and Ramchandran, Kannan and others},
  journal={Advances in Neural Information Processing Systems},
  volume={38},
  year={2026}
}

@inproceedings{xie2026agentsynth,
  title={Agentsynth: Scalable task generation for generalist computer-use agents},
  author={Xie, Jingxu and Xu, Dylan and Zhao, Xuandong and Song, Dawn},
  booktitle={International Conference on Learning Representations},
  volume={2026},
  pages={155149--155177},
  year={2026}
}

@article{felicia2026stepshield,
  title={StepShield: When, not whether to intervene on rogue agents},
  author={Felicia, Gloria and Sasindran, Zitha and He, Jinfeng and Eniolade, Michael and Kumar, Hemant and Angati, Milan Hussain},
  journal={arXiv preprint arXiv:2601.22136},
  year={2026}
}

@article{gu2026safe,
  title={Safe: Multitask failure detection for vision-language-action models},
  author={Gu, Qiao and Ju, Yuanliang and Sun, Shengxiang and Gilitschenski, Igor and Nishimura, Haruki and Itkina, Masha and Shkurti, Florian},
  journal={Advances in Neural Information Processing Systems},
  volume={38},
  pages={40041--40076},
  year={2026}
}

@article{timans2023max,
  title={Max-rank: Efficient multiple testing for conformal prediction},
  author={Timans, Alexander and Straehle, Christoph-Nikolas and Sakmann, Kaspar and Naesseth, Christian A and Nalisnick, Eric},
  journal={arXiv preprint arXiv:2311.10900},
  year={2023}
}

@article{malinin2020uncertainty,
  title={Uncertainty estimation in autoregressive structured prediction},
  author={Malinin, Andrey and Gales, Mark},
  journal={arXiv preprint arXiv:2002.07650},
  year={2020}
}

@article{ren2022out,
  title={Out-of-distribution detection and selective generation for conditional language models},
  author={Ren, Jie and Luo, Jiaming and Zhao, Yao and Krishna, Kundan and Saleh, Mohammad and Lakshminarayanan, Balaji and Liu, Peter J},
  journal={arXiv preprint arXiv:2209.15558},
  year={2022}
}

@article{inan2023llama,
  title={Llama guard: Llm-based input-output safeguard for human-ai conversations},
  author={Inan, Hakan and Upasani, Kartikeya and Chi, Jianfeng and Rungta, Rashi and Iyer, Krithika and Mao, Yuning and Tontchev, Michael and Hu, Qing and Fuller, Brian and Testuggine, Davide and others},
  journal={arXiv preprint arXiv:2312.06674},
  year={2023}
}

@article{romer2026failure,
  title={Failure prediction at runtime for generative robot policies},
  author={R{\"o}mer, Ralf and Kobras, Adrian and Worbis, Luca and Schoellig, Angela},
  journal={Advances in Neural Information Processing Systems},
  volume={38},
  pages={7631--7670},
  year={2026}
}

@article{wu2020adversarial,
  title={Adversarial weight perturbation helps robust generalization},
  author={Wu, Dongxian and Xia, Shu-Tao and Wang, Yisen},
  journal={Advances in neural information processing systems},
  volume={33},
  pages={2958--2969},
  year={2020}
}

@article{hu2021lora,
  title={Lora: Low-rank adaptation of large language models},
  author={Hu, Edward J and Shen, Yelong and Wallis, Phillip and Allen-Zhu, Zeyuan and Li, Yuanzhi and Wang, Shean and Wang, Lu and Chen, Weizhu},
  journal={arXiv preprint arXiv:2106.09685},
  year={2021}
}

@article{wang2023out,
  title={Out-of-distribution detection with implicit outlier transformation},
  author={Wang, Qizhou and Ye, Junjie and Liu, Feng and Dai, Quanyu and Kalander, Marcus and Liu, Tongliang and Hao, Jianye and Han, Bo},
  journal={arXiv preprint arXiv:2303.05033},
  year={2023}
}

@inproceedings{xu2025agenttrek,
  title={Agenttrek: Agent trajectory synthesis via guiding replay with web tutorials},
  author={Xu, Yiheng and Lu, Dunjie and Shen, Zhennan and Wang, Junli and Wang, Zekun and Mao, Yuchen and Xiong, Caiming and Yu, Tao},
  booktitle={International Conference on Learning Representations},
  volume={2025},
  pages={79822--79843},
  year={2025}
}

@article{bonagiri2026causalflow,
  title={CausalFlow: Causal Attribution and Counterfactual Repair for LLM Agent Failures},
  author={Bonagiri, Akash and Borkar, Devang and Anderias, Gerard Janno and Rafatirad, Setareh and Homayoun, Houman},
  journal={arXiv preprint arXiv:2605.25338},
  year={2026}
}

@inproceedings{kong2026aegis,
  title={Aegis: Automated error generation and attribution for multi-agent systems},
  author={Kong, Fanqi and Zhang, Ruijie and Yin, Huaxiao and Zhang, Guibin and Zhang, Xiaofei and Chen, Ziang and Zhang, Zhaowei and Zhang, Xiaoyuan and Zhu, Song-Chun and Feng, Xue},
  booktitle={International Conference on Learning Representations},
  volume={2026},
  pages={52994--53028},
  year={2026}
}

@article{huang2026prefixguard,
  title={PrefixGuard: From LLM-Agent Traces to Online Failure-Warning Monitors},
  author={Huang, Xinmiao and Hu, Jinwei and Roy, Rajarshi and Wu, Changshun and Dong, Yi and Huang, Xiaowei},
  journal={arXiv preprint arXiv:2605.06455},
  year={2026}
}

@article{baidya2026evidence,
  title={When Evidence is Sparse: Weakly Supervised Early Failure Alerting in Dialogs and LLM-Agent Trajectories},
  author={Baidya, Avinash and Liang, Xinran and Guo, Ruocheng and Gao, Xiang and Das, Kamalika},
  journal={arXiv preprint arXiv:2606.05414},
  year={2026}
}

@inproceedings{oh2026uncertainty,
  title={Uncertainty quantification in llm agents: Foundations, emerging challenges, and opportunities},
  author={Oh, Changdae and Park, Seongheon and Kim, To Eun and Li, Jiatong and Li, Wendi and Yeh, Samuel and Du, Sean and Hassani, Hamed and Bogdan, Paul and Song, Dawn and others},
  booktitle={Proceedings of the 64th Annual Meeting of the Association for Computational Linguistics (Volume 1: Long Papers)},
  pages={16219--16250},
  year={2026}
}

@article{zhang2026agentforesight,
  title={Agentforesight: Online auditing for early failure prediction in multi-agent systems},
  author={Zhang, Boxuan and Zhu, Jianing and Shi, Zeru and Liu, Dongfang and Tang, Ruixiang},
  journal={arXiv preprint arXiv:2605.08715},
  year={2026}
}

@article{li2026traces,
  title={TRACES: Proactive Safety Auditing for Multi-Turn LLM Agents via Trajectory-State Modeling},
  author={Li, Jiaqian and Li, Yanshu and Zhang, Boxuan and Tang, Ruixiang and Huang, Kuan-Hao},
  journal={arXiv preprint arXiv:2605.27690},
  year={2026}
}

@inproceedings{ruan2024identifying,
  title={Identifying the risks of LM agents with an LM-emulated sandbox},
  author={Ruan, Yangjun and Dong, Honghua and Wang, Andrew and Pitis, Silviu and Zhou, Yongchao and Ba, Jimmy and Dubois, Yann and Maddison, Chris and Hashimoto, Tatsunori},
  booktitle={International Conference on Learning Representations},
  volume={2024},
  pages={27031--27098},
  year={2024}
}

@article{park2026hide,
  title={Hide-and-Seek in Trajectories: Discovering Failure Signals for VLA Runtime Monitoring},
  author={Park, Seongheon and Li, Wendi and Oh, Changdae and Yeh, Samuel and Kira, Zsolt and Hagenow, Michael and Li, Sharon},
  journal={arXiv preprint arXiv:2605.30834},
  year={2026}
}

@inproceedings{wu2024autogen,
  title={Autogen: Enabling next-gen LLM applications via multi-agent conversations},
  author={Wu, Qingyun and Bansal, Gagan and Zhang, Jieyu and Wu, Yiran and Li, Beibin and Zhu, Erkang and Jiang, Li and Zhang, Xiaoyun and Zhang, Shaokun and Liu, Jiale and others},
  booktitle={First conference on language modeling},
  year={2024}
}

@inproceedings{hong2023metagpt,
  title={MetaGPT: Meta programming for a multi-agent collaborative framework},
  author={Hong, Sirui and Zhuge, Mingchen and Chen, Jonathan and Zheng, Xiawu and Cheng, Yuheng and Wang, Jinlin and Zhang, Ceyao and Wang, Zili and Yau, Steven Ka Shing and Lin, Zijuan and others},
  booktitle={The twelfth international conference on learning representations},
  year={2023}
}

@article{li2023camel,
  title={Camel: Communicative agents for" mind" exploration of large language model society},
  author={Li, Guohao and Hammoud, Hasan and Itani, Hani and Khizbullin, Dmitrii and Ghanem, Bernard},
  journal={Advances in neural information processing systems},
  volume={36},
  pages={51991--52008},
  year={2023}
}

@inproceedings{yao2022react,
  title={React: Synergizing reasoning and acting in language models},
  author={Yao, Shunyu and Zhao, Jeffrey and Yu, Dian and Du, Nan and Shafran, Izhak and Narasimhan, Karthik R and Cao, Yuan},
  booktitle={The eleventh international conference on learning representations},
  year={2022}
}

@article{zhang2025agentracer,
  title={AgenTracer: Who Is Inducing Failure in the LLM Agentic Systems?},
  author={Zhang, Guibin and Wang, Junhao and Chen, Junjie and Zhou, Wangchunshu and Wang, Kun and Yan, Shuicheng},
  journal={arXiv preprint arXiv:2509.03312},
  year={2025}
}

@article{zhang2025agent,
  title={Which agent causes task failures and when? on automated failure attribution of llm multi-agent systems},
  author={Zhang, Shaokun and Yin, Ming and Zhang, Jieyu and Liu, Jiale and Han, Zhiguang and Zhang, Jingyang and Li, Beibin and Wang, Chi and Wang, Huazheng and Chen, Yiran and others},
  journal={arXiv preprint arXiv:2505.00212},
  year={2025}
}

@article{zhu2025llm,
  title={Where llm agents fail and how they can learn from failures},
  author={Zhu, Kunlun and Liu, Zijia and Li, Bingxuan and Tian, Muxin and Yang, Yingxuan and Zhang, Jiaxun and Han, Pengrui and Xie, Qipeng and Cui, Fuyang and Zhang, Weijia and others},
  journal={arXiv preprint arXiv:2509.25370},
  year={2025}
}

@article{zhang2025dive,
  title={Dive into the Agent Matrix: A Realistic Evaluation of Self-Replication Risk in LLM Agents},
  author={Zhang, Boxuan and Yu, Yi and Guo, Jiaxuan and Shao, Jing},
  journal={arXiv preprint arXiv:2509.25302},
  year={2025}
}

@article{yang2025qwen3,
  title={Qwen3 technical report},
  author={Yang, An and Li, Anfeng and Yang, Baosong and Zhang, Beichen and Hui, Binyuan and Zheng, Bo and Yu, Bowen and Gao, Chang and Huang, Chengen and Lv, Chenxu and others},
  journal={arXiv preprint arXiv:2505.09388},
  year={2025}
}

@article{grattafiori2024llama,
  title={The llama 3 herd of models},
  author={Grattafiori, Aaron and Dubey, Abhimanyu and Jauhri, Abhinav and Pandey, Abhinav and Kadian, Abhishek and Al-Dahle, Ahmad and Letman, Aiesha and Mathur, Akhil and Schelten, Alan and Vaughan, Alex and others},
  journal={arXiv preprint arXiv:2407.21783},
  year={2024}
}

@article{schick2023toolformer,
  title={Toolformer: Language models can teach themselves to use tools},
  author={Schick, Timo and Dwivedi-Yu, Jane and Dess{\`\i}, Roberto and Raileanu, Roberta and Lomeli, Maria and Hambro, Eric and Zettlemoyer, Luke and Cancedda, Nicola and Scialom, Thomas},
  journal={Advances in neural information processing systems},
  volume={36},
  pages={68539--68551},
  year={2023}
}

@article{zhuge2024language,
  title={Language agents as optimizable graphs},
  author={Zhuge, Mingchen and Wang, Wenyi and Kirsch, Louis and Faccio, Francesco and Khizbullin, Dmitrii and Schmidhuber, J{\"u}rgen},
  journal={arXiv preprint arXiv:2402.16823},
  year={2024}
}

@article{qian2026behind,
  title={The Why Behind the Action: Unveiling Internal Drivers via Agentic Attribution},
  author={Qian, Chen and Wang, Peng and Liu, Dongrui and Yang, Junyao and Guo, Dadi and Tang, Ling and Mei, Jilin and Ren, Qihan and Shao, Shuai and Liu, Yong and others},
  journal={arXiv preprint arXiv:2601.15075},
  year={2026}
}

@inproceedings{zhang2025cot,
  title={Cot-uq: Improving response-wise uncertainty quantification in llms with chain-of-thought},
  author={Zhang, Boxuan and Zhang, Ruqi},
  booktitle={Findings of the Association for Computational Linguistics: ACL 2025},
  pages={26114--26133},
  year={2025}
}

@article{liu2026agentdog,   title={Agentdog: A diagnostic guardrail framework for ai agent safety and security},   author={Liu, Dongrui and Ren, Qihan and Qian, Chen and Shao, Shuai and Xie, Yuejin and Li, Yu and Yang, Zhonghao and Luo, Haoyu and Wang, Peng and Liu, Qingyu and others},   journal={arXiv preprint arXiv:2601.18491},   year={2026} }

@article{zheng2026stepguard,   title={StepGuard: Learning Step-Level Guardrails with Scalable Supervision and Safety-Utility Balancing},   author={Zheng, Zhijie and Li, Yu and Qian, Chen and Fu, Yuqian and Fu, Yanwei and Sheng, Lu and Shao, Jing and Liu, Dongrui},   journal={arXiv preprint arXiv:2608.24777},   year={2026} }

@article{yeh2026tracing,   title={Tracing Agentic Failure from the Flow of Success},   author={Yeh, Samuel and Zhu, Yiwen and Deep, Shaleen and Li, Sharon},   journal={arXiv preprint arXiv:2607.12747},   year={2026} }

@inproceedings{pmlr-v306-sun26ba,
  title={{AgentHijack}: Benchmarking Computer Use Agent Robustness to Common Environment Corruptions},
  author={Sun, Jingwei and Zhu, Jianing and Li, Yuanyi and Liu, Tongliang and Hu, Xia and Han, Bo},
  booktitle={Proceedings of the 43rd International Conference on Machine Learning},
  series={Proceedings of Machine Learning Research},
  volume={306},
  pages={117524--117556},
  publisher={PMLR},
  year={2026}
}

@article{sun2026rethinking,
  title={Rethinking How to Remember: Beyond Atomic Facts in Lifelong {LLM} Agent Memory},
  author={Sun, Jingwei and Zhu, Jianing and Yao, Jiangchao and Liu, Tongliang and Han, Bo},
  journal={arXiv preprint arXiv:2605.19952},
  year={2026}
}

@article{zhu2026your,
  title={Your Agents Are Aging Too: Agent Lifespan Engineering for Deployed Systems},
  author={Zhu, Jianing and Ro, Yeonju and Robertson, John and Wang, Kevin and Li, Junbo and Vikalo, Haris and Akella, Aditya and Wang, Zhangyang},
  journal={arXiv preprint arXiv:2605.26302},
  year={2026}
}

@article{li2023deepinception,
  title={Deepinception: Hypnotize large language model to be jailbreaker},
  author={Li, Xuan and Zhou, Zhanke and Zhu, Jianing and Yao, Jiangchao and Liu, Tongliang and Han, Bo},
  journal={arXiv preprint arXiv:2311.03191},
  year={2023}
}

\clearpage
\appendix

\section{Additional Method Details}
\label{app:method-details}

\subsection{Training Algorithm and Solver Details}
\label{app:solver}
Algorithm~\ref{alg:dream2discern} specifies the alternating updates after the supervised initialization in Appendix~\ref{app:exp-training}.
It uses the batch loss under the effective-operator perturbations in Section~\ref{sec:steer-to-confuse} and the scaled perturbation norm in Equation~\ref{eq:ambiguity-set}.
Here $\mathcal P_r$ denotes truncated SVD retaining at most $r$ leading singular components, and $\operatorname{Eval}$ computes a loss without gradients through the same numerical path for the unperturbed auditor and each candidate.

\begin{algorithm}[htbp]
\caption{Dream2Discern: learning from hard latent variants}
\label{alg:dream2discern}
\small
\begin{algorithmic}[1]
\Require Training-prefix index $\mathcal D$; initialized auditor $\Theta=(a,\psi)$;
  frozen backbone $\phi_0$; suffix query/value projections $\mathcal S$;
  radius $\gamma$; rank cap $r$; iterations $K$; batch size $2m$.
\State Fix $\omega_j\gets\lVert W_j^0\rVert_F$ for $j\in\mathcal S$;
  $\epsilon_{\mathrm{num}}\gets10^{-12}$.
\For{$k=1,\ldots,K$}
  \State Form $\mathcal B$ by drawing $m$ prefixes per class from $\mathcal D$ with replacement:
    parent group $\to$ trajectory $\to$ indexed prefix, uniformly at each level.
  \Statex \hspace{\algorithmicindent}\textbf{Dream: hold $\Theta$ fixed and select a shared perturbation.}
  \State $G_j\gets\left.\nabla_{V_j}\mathcal L_{\mathcal B}(\Theta,V)\right|_{V=0}$;
    $C_j\gets\omega_jG_j$ for all $j\in\mathcal S$.
  \State $d\gets\max\!\left\{\left(\sum_{j\in\mathcal S}\lVert C_j\rVert_F^2\right)^{1/2},
      \epsilon_{\mathrm{num}}\right\}$.
  \State $\overline V_{\mathcal B,j}\gets\mathcal P_r(\gamma\omega_j C_j/d)$
    for all $j\in\mathcal S$.
  \State $\displaystyle\overline V_{\mathcal B}\gets\overline V_{\mathcal B}\,
    \min\!\left\{1,\frac{\gamma}{\max\{\lVert\overline V_{\mathcal B}\rVert_\omega,\epsilon_{\mathrm{num}}\}}\right\}$;
    $\widehat V_{\mathcal B}\gets0$.
  \State $\ell_0\gets\operatorname{Eval}(\mathcal L_{\mathcal B}(\Theta,0))$.
  \For{$\lambda\in(1,2^{-1},\ldots,2^{-5})$ in order}
    \If{$\operatorname{Eval}(\mathcal L_{\mathcal B}(\Theta,\lambda\overline V_{\mathcal B}))>\ell_0$}
      \State $\widehat V_{\mathcal B}\gets\lambda\overline V_{\mathcal B}$; \textbf{break}.
    \EndIf
  \EndFor
  \Statex \hspace{\algorithmicindent}\textbf{Discern: learn from the returned variants.}
  \State $g_\Theta\gets\nabla_\Theta\mathcal L_{\mathcal B}
    (\Theta,\operatorname{sg}[\widehat V_{\mathcal B}])$.
  \State Remove the temporary perturbation; $\Theta\gets
    \operatorname{AdamW}(\Theta,\operatorname{clip}_{1}(g_\Theta))$.
\EndFor
\State \Return Auditor $f_\Theta$ with $V=0$.
\end{algorithmic}
\end{algorithm}

\noindent \textbf{Solver details.}
Each batch starts from $V=0$ and uses one normalized gradient step followed by SVD truncation and a global radial projection.
Normalization uses the full scaled gradient norm with the numerical floor in Algorithm~\ref{alg:dream2discern}, rather than the retained norm $g_r$ in Proposition~\ref{prop:linear-oracle}.
For the experimental settings in Appendix~\ref{app:exp-training}, the resulting candidate satisfies, in exact arithmetic,
\[
    \overline V_{\mathcal B,j}=\gamma\omega_j\frac{C_{j,r}}{d},
    \qquad
    \lVert\overline V_{\mathcal B}\rVert_\omega
    =\gamma\frac{g_r}{d}\leq\gamma,
\]
where $C_{j,r}$ uses the singular components retained by the algorithm.
Thus the candidate follows a maximizing direction of the linearized loss, but can have a smaller radius than the nominal linear solution; it is zero when $g_r=0$.
The subsequent acceptance test selects the first of six tested points that raises the measured batch loss, returning zero if all are rejected.

\noindent \textbf{Discern update.}
The outer backward pass holds $\widehat V_{\mathcal B}$ constant and computes gradients only for the persistent adapters and prediction head.
We remove the temporary perturbation, clip the global gradient norm at 1, and apply AdamW; optimizer settings are given in Appendix~\ref{app:exp-training}.
A zero returned perturbation yields the unperturbed batch update.
The saved auditor contains the learned adapters and head and uses $V=0$ at runtime.

\noindent \textbf{Perturbation space.}
We perturb representation operators to couple the latent changes of different prefixes through a shared transformation.
This follows the model-wise perspective of adversarial weight perturbation~\citep{wu2020adversarial}, which selects a weight perturbation using the loss over a batch.
In our setting, Equation~\ref{eq:implicit-transformation} gives the direct additive change $H_{q,j}^{V}V_j^{\top}$ at projection $j$ for fixed incoming states.
The shared operator couples these changes across tokens and prefixes, while their coefficients depend on the incoming states and their local output directions are restricted by the rank cap.
This defines a structured search space for hard latent variants: query and value perturbations change the attention computation under a fixed prediction head.

\subsection{Runtime Threshold Calibration}
\label{app:threshold}

\noindent \textbf{Threshold-fitting data.}
For \method{}, we sample 40\% of the parent groups represented in the training-prefix index using seed 0 and retain their successful trajectories with at least one nonterminal prefix.
These trajectories come from the auditor's training partition and are distinct from the evaluation trajectories.
With the trained auditor fixed at $V=0$, we score every nonterminal prefix, indexed by $t=1,\ldots,T_i-1$, for each of the $n_{\mathrm{cal}}$ selected trajectories.

\noindent \textbf{Stepwise reference distributions.}
Let $\mathcal C_t$ be the multiset of calibration scores at step $t$, with $m_t=|\mathcal C_t|$.
Because every nonterminal prefix is scored, $m_t$ is nonincreasing.
For $n_{\mathrm{cal}}\ge20$, define $t_\star=\max\{t:m_t\ge20\}$ and $b(t)=\min\{t,t_\star\}$, carrying the last supported reference distribution to later steps.
For sorted scores $s^{\mathrm{cal}}_{j,(1)}\le\cdots\le s^{\mathrm{cal}}_{j,(m_j)}$ in $\mathcal C_j$, define
\begin{equation}
    \begin{aligned}
        F_t(x)&=\frac{1}{m_{b(t)}}\sum_{s\in\mathcal C_{b(t)}}\mathbf 1\{s\le x\},\\
        Q_t^\uparrow(u)&=s^{\mathrm{cal}}_{b(t),(1+\lceil(m_{b(t)}-1)u\rceil)},\qquad u\in[0,1].
    \end{aligned}
    \label{eq:runtime-reference}
\end{equation}
The CDF includes ties, and $Q_t^\uparrow$ uses the \texttt{higher} quantile convention.

\noindent \textbf{Trajectory maxima and thresholds.}
We summarize each trajectory by its maximum empirical percentile over nonterminal steps, following the rank-aggregation idea of max-rank~\citep{timans2023max}.
For nominal level $\alpha\in(0,1)$, define
\begin{equation}
    \begin{aligned}
        R_i&=\max_{1\le t<T_i}F_t(s_{i,t}),\qquad
        k=\big\lceil(n_{\mathrm{cal}}+1)(1-\alpha)\big\rceil,\\
        \zeta_t&=
        \begin{cases}
            Q_t^\uparrow\big(R_{(k)}\big),&k\le n_{\mathrm{cal}},\\
            +\infty,&k>n_{\mathrm{cal}},
        \end{cases}
    \end{aligned}
    \label{eq:rank-band}
\end{equation}
where $R_{(k)}$ is the $k$-th smallest trajectory maximum.
All nonterminal scores enter $R_i$, including those after $t_\star$ evaluated against the carried reference distribution.
The resulting thresholds remain fixed at runtime, with $\zeta_t=\zeta_{t_\star}$ for $t>t_\star$ and failure predicted when $s_{i,t}\ge\zeta_t$ as in Equation~\ref{eq:runtime-decision}.
We use $\alpha=0.05$ for the main \method{} results; at the $\alpha=1$ endpoint of the threshold sweep, we set $\zeta_t=Q_t^\uparrow(0)$.
The same training trajectories determine both $F_t$ and $R_i$, so $\alpha$ is a nominal parameter of this empirical calibration procedure, not a finite-sample false-alarm guarantee.

\section{Theoretical Analysis}
\label{app:proofs}

\subsection{Proof of Proposition~\ref{prop:linear-oracle}}
\label{app:proof-linear-oracle}
\begin{proof}
For any $V\in\mathcal U_{\gamma,r}$, set $U_j=V_j/\omega_j$.
Since $\omega_j>0$, this rescaling preserves rank and gives
\[
    \sum_{j\in\mathcal S}\lVert U_j\rVert_F^2\leq\gamma^2,
    \qquad
    \sum_{j\in\mathcal S}\langle G_j,V_j\rangle
    =\sum_{j\in\mathcal S}\langle C_j,U_j\rangle.
\]
Order the singular values nonincreasingly and pad them with zeros when necessary.
Since $\operatorname{rank}(U_j)\leq r$, von Neumann's trace inequality followed by Cauchy--Schwarz yields
\[
    \begin{aligned}
        \langle C_j,U_j\rangle
        &\leq\sum_{k=1}^{r}\sigma_k(C_j)\,\sigma_k(U_j)\\
        &\leq
        \left(\sum_{k=1}^{r}\sigma_k(C_j)^2\right)^{1/2}
        \left(\sum_{k=1}^{r}\sigma_k(U_j)^2\right)^{1/2}
        =\lVert C_{j,r}\rVert_F\,\lVert U_j\rVert_F.
    \end{aligned}
\]
Summing over projections and applying Cauchy--Schwarz once more gives
\[
    \begin{aligned}
        \sum_{j\in\mathcal S}\langle C_j,U_j\rangle
        &\leq\sum_{j\in\mathcal S}\lVert C_{j,r}\rVert_F\,\lVert U_j\rVert_F\\
        &\leq
        \left(\sum_{j\in\mathcal S}\lVert C_{j,r}\rVert_F^2\right)^{1/2}
        \left(\sum_{j\in\mathcal S}\lVert U_j\rVert_F^2\right)^{1/2}
        \leq\gamma g_r.
    \end{aligned}
\]

For $g_r>0$, choose $U_j=\gamma C_{j,r}/g_r$.
These matrices have rank at most $r$ and satisfy $\sum_j\lVert U_j\rVert_F^2=\gamma^2$.
The truncated SVD also gives $\langle C_j,C_{j,r}\rangle=\lVert C_{j,r}\rVert_F^2$, so
\[
    \sum_{j\in\mathcal S}\langle C_j,U_j\rangle
    =\frac{\gamma}{g_r}\sum_{j\in\mathcal S}\lVert C_{j,r}\rVert_F^2
    =\gamma g_r.
\]
Thus, mapping back to $V_j=\omega_jU_j$ gives the feasible maximizer in Equation~\ref{eq:linear-oracle}, which attains the upper bound.
If $g_r=0$, the same bound is zero and is attained at $V=0$.
\end{proof}

\subsection{Proof of Proposition~\ref{prop:accepted-challenge}}
\label{app:proof-accepted-challenge}
\begin{proof}
Fix $\Theta$ and $\mathcal B$ as in the proposition.
Its remainder assumption states that, for every $V$ on the segment from $0$ to $\overline V_{\mathcal B}$,
\[
    \left|\mathcal L_{\mathcal B}(\Theta,V)
    -\mathcal L_{\mathcal B}(\Theta,0)
    -\sum_{j\in\mathcal S}\langle G_j,V_j\rangle\right|
    \leq \frac{\beta}{2}\lVert V\rVert_\omega^2.
\]
If $g_r=0$, every retained gradient block is zero, so the candidate and returned perturbation are zero; Equation~\ref{eq:first-order-hard-loss} holds with both sides zero.
For $g_r>0$, use the singular components retained by the algorithm to define $C_{j,r}$.
This choice does not change $g_r$.
Before radial rescaling, the truncated blocks in Algorithm~\ref{alg:dream2discern} are $\gamma\omega_j C_{j,r}/d$, where $d>0$ is its normalization denominator.
Radial rescaling and the accepted step fraction apply a common nonnegative factor to these blocks.
Since $(\omega_j C_{j,r})_{j\in\mathcal S}$ has scaled norm $g_r$, the returned point has the form
\[
    \widehat V_{\mathcal B,j}
    =\rho_{\mathcal B}\,\omega_j\frac{C_{j,r}}{g_r},
    \qquad j\in\mathcal S,
\]
including $\rho_{\mathcal B}=0$ when the candidate is rejected.
The truncated SVD satisfies $\langle C_j,C_{j,r}\rangle=\lVert C_{j,r}\rVert_F^2$.
Together with $C_j=\omega_jG_j$, this gives
\[
    \begin{aligned}
        \sum_{j\in\mathcal S}\langle G_j,\widehat V_{\mathcal B,j}\rangle
        &=\frac{\rho_{\mathcal B}}{g_r}
          \sum_{j\in\mathcal S}\langle C_j,C_{j,r}\rangle\\
        &=\frac{\rho_{\mathcal B}}{g_r}
          \sum_{j\in\mathcal S}\lVert C_{j,r}\rVert_F^2
          =\rho_{\mathcal B}g_r.
    \end{aligned}
\]
By Equation~\ref{eq:accepted-inner}, $\widehat V_{\mathcal B}=\lambda_{\mathcal B}\overline V_{\mathcal B}$ with $\lambda_{\mathcal B}\in[0,1]$, so the returned point lies on the required segment.
The remainder bound and the inner-product identity then yield Equation~\ref{eq:first-order-hard-loss}.
\end{proof}

\noindent \textbf{A sufficient condition for accepting the full candidate.}
For a nonzero candidate, $g_r>0$, and the same inner-product identity and remainder bound give
\[
    \mathcal L_{\mathcal B}(\Theta,\overline V_{\mathcal B})
    -\mathcal L_{\mathcal B}(\Theta,0)
    \geq \lVert\overline V_{\mathcal B}\rVert_\omega g_r
      -\frac{\beta}{2}\lVert\overline V_{\mathcal B}\rVert_\omega^2.
\]
Consequently, $\beta\lVert\overline V_{\mathcal B}\rVert_\omega<2g_r$ ensures a strictly positive loss increase, so the first tested point is accepted in exact arithmetic.
This condition also covers $\beta=0$.
If all six tested points are rejected, the search returns zero; a smaller untested step may still increase the loss.

\section{Additional Experimental Setups}
\label{app:setup}

\subsection{Datasets and Evaluation Partitions}
\label{app:exp-data}

\noindent \textbf{Dataset scope and labels.}
\textsc{AFTraj-2K}~\citep{zhang2026agentforesight} contains successful and failed multi-agent executions across mathematical reasoning, coding, and tool-based tasks.
We use its recorded outcomes, with task failure mapped to $Y_i=1$ and success to $Y_i=0$.
For StepShield~\citep{felicia2026stepshield}, we use the paired clean/rogue trajectory subset in our evaluation corpus, with policy violations (failure) mapped to $Y_i=1$ and clean (success) trajectories to $Y_i=0$.
Table~\ref{tab:exp-data} distinguishes the available training trajectories from those retained after constructing the training-prefix index.

\begin{table}[ht]
\centering
\small
\caption{Dataset statistics. Training-index rows count prefixes; test counts refer to trajectories.}
\label{tab:exp-data}
\begin{tabular}{lrr}
\toprule
 & \textsc{AFTraj-2K} & StepShield \\
\midrule
Available training trajectories & 1,944 & 1,278 \\
Trajectories in the training index & 1,821 & 1,253 \\
Training-index rows & 8,454 & 7,828 \\
Test trajectories & 332 & 216 \\
Test trajectories with $Y_i=1$ & 163 & 108 \\
Test trajectories with $Y_i=0$ & 169 & 108 \\
\bottomrule
\end{tabular}
\end{table}

\noindent \textbf{Partitions and prefix access.}
We retain the recorded \textsc{AFTraj-2K} train/test assignment.
Grouping trajectories by the source-task stems of their identifiers yields 19 parent groups shared across the partitions, covering 26 test trajectories.
For the StepShield subset, the split contains 639 training pairs and 108 test pairs, with both members of a pair assigned to the same partition.
Training prefixes are selected using the annotations and balancing procedure in Appendix~\ref{app:prefix-construction}.
The training index includes terminal prefixes, whereas evaluation uses all nonterminal prefixes $1\le t<T_i$.
Auditor inputs contain the task context and observed interactions but exclude outcome labels and decisive-step annotations; \textsc{AFTraj-2K} inputs also omit the stored thought field.

\subsection{Training Prefix Construction}
\label{app:prefix-construction}
For the annotated training corpora, we select prefixes offline from completed trajectories.
For a failed trajectory $i$, eligible prefixes end at $t=K_i,\ldots,T_i$, where $K_i$ is its annotated decisive step.
For a successful trajectory, we sample a starting step from the empirical distribution of training failure annotations, resampling starts at or beyond $T_i$ and using a valid fallback when necessary.
Candidate prefixes extend from this start through $T_i$, and each retains its trajectory's terminal outcome as the prediction target.
We subsample equal numbers of success and failure prefixes within each combination of ending step, offset from the selected start, and terminal status.
This balancing applies to the stored prefix pool; Algorithm~\ref{alg:dream2discern} subsequently draws class-balanced batches through the group, trajectory, and prefix hierarchy.
The annotations guide selection only and are excluded from auditor inputs.
Prefix selection precedes token-budget truncation.

\subsection{Auditor Training and Reporting}
\label{app:exp-training}

\noindent \textbf{Parameterization and input length.}
For all four backbones, we use the auditor parameterization in Section~\ref{sec:steer-to-confuse}.
Query and value projections after the middle Transformer block receive LoRA adapters of rank 8 with $\alpha_{\mathrm{LoRA}}=16$, giving an adapter multiplier of $\alpha_{\mathrm{LoRA}}/8=2$.
We use last-token readout by default: the final-layer hidden state of the last non-padding token is fed to a two-layer MLP with a 256-dimensional hidden layer and GELU activation.
The input budget is 2,048 tokens; longer serialized prefixes retain the first 256 and last 1,792 tokens.
Table~\ref{tab:exp-training} lists the optimization settings.
We fix one perturbation radius per backbone across both benchmarks, taking the one whose \emph{Dream} step most reduces the auditor's signed margin on training batches, as listed in Table~\ref{tab:exp-radii}.

\noindent \textbf{Training budget and checkpoint selection.}
For each benchmark and backbone, we first fit the adapters and prediction head for 100 supervised steps.
We then apply \textit{Dream} and \textit{Discern} for 1,057 steps on \textsc{AFTraj-2K} or 979 steps on StepShield and evaluate the final checkpoint.
With logical batches of 32 prefixes, these budgets correspond to approximately four times the size of the respective training-prefix index.
Because the hierarchical sampler draws with replacement, an epoch denotes this sampling budget, not a pass visiting every prefix exactly once.
The threshold-fitting procedure and its data source are specified in Appendix~\ref{app:threshold}.
Auditor training uses NVIDIA B200 GPUs with PyTorch 2.8.0 and CUDA 12.8.

\noindent \textbf{Random seeds and uncertainty.}
Tables~\ref{tab:main_results} and~\ref{tab:cross-family} report means and sample standard deviations over training seeds $0,1,2$ for \method{}, SAFE, Hide-and-Seek, and TRACES across all four backbones.
For \method{}, runs within each benchmark and backbone share the same 100-step, seed-0 warm-start checkpoint and initialize a fresh optimizer and random state for subsequent training.
The data split and threshold-fitting seed remain fixed across runs.

\begin{table}[ht]
\centering
\small
\caption{Training settings for \method{}.}
\label{tab:exp-training}
\begin{tabular}{ll}
\toprule
Setting & Value \\
\midrule
Optimizer & AdamW \\
Adapter and head learning rates & $10^{-4}$ \\
Weight decay & $0$ \\
Gradient clipping norm & $1$ \\
Logical batch size & 32 prefixes, 16 per outcome class \\
Microbatch size / accumulation steps & $8$ / $4$ \\
Supervised warm-start steps & 100 \\
Subsequent steps, \textsc{AFTraj-2K} / StepShield & 1,057 / 979 \\
Perturbation rank cap & $r=8$ \\
Gradient evaluations per \textit{Dream} search & One at the unperturbed auditor \\
Candidate evaluation budget & At most six, as in Algorithm~\ref{alg:dream2discern} \\
Nominal threshold-fitting level & $\alpha=0.05$ \\
\bottomrule
\end{tabular}
\end{table}

\begin{table}[ht]
\centering
\small
\caption{Backbone-specific perturbation radii. Each value is used on both benchmarks.}
\label{tab:exp-radii}
\begin{tabular}{lc}
\toprule
Auditor backbone & $\gamma$ \\
\midrule
Qwen3-4B & $3\times10^{-4}$ \\
Qwen3-8B & $3\times10^{-3}$ \\
Llama-3.2-3B-Instruct & $3\times10^{-3}$ \\
Llama-3.1-8B-Instruct & $3\times10^{-4}$ \\
\bottomrule
\end{tabular}
\end{table}

\subsection{Baseline Implementations}
\label{app:exp-baselines}

\noindent \textbf{Open-weight judges.}
The Qwen and Llama instruction models are prompted to judge the observed prefix without additional task-specific fitting.
Their failure scores are the \texttt{FAIL} probabilities from a softmax over the \texttt{FAIL} and \texttt{OK} answer logits.
We use a 2,048-token prefix budget, preserving the task context and recent interactions when truncation is needed.
LlamaGuard3-8B uses a task-failure prompt on \textsc{AFTraj-2K} and its standard safety taxonomy on StepShield.
For each benchmark, we report the complete LlamaGuard variant winning more of the four test metrics among the native and task-specific prompts.

\noindent \textbf{Closed-source judges.}
The closed-source comparisons use the recorded outputs of GPT-5.4, Claude-Sonnet-4.6, Grok-4.3, and Kimi-K2.7-Code.
For these models, the returned continuous risk score is used for AUROC and FPR95, while a returned failure label supplies the binary prediction used for wACC and TWA.
These decisions therefore use the models' reported labels rather than the fitted threshold band used for the other rows.
Unparseable responses are retained as no-alarm predictions and assigned scores below all valid responses in that run.
Each closed-source row represents one recorded evaluation run.
The output budget is 4,096 tokens.
On \textsc{AFTraj-2K}, GPT-5.4 also participated in planning fault injections during dataset construction.

\noindent \textbf{AgentDebug~\citep{zhu2025llm} and AgentForesight~\citep{zhang2026agentforesight}.}
We adapt AgentDebug to Llama-3.1-8B-Instruct and apply its two-stage diagnosis to the observed prefix, removing the prompt's premise that failure has already occurred.
An additional scoring stage converts the diagnosis into a failure score using the \texttt{FAIL} probability from a softmax over the \texttt{FAIL}/\texttt{OK} logits.
Its prompt allocates 1,536 tokens to context and 3,072 tokens to the interaction block.
For AgentForesight, we use the released AgentForesight-7B checkpoint with its native prompt and greedy decoding, without additional training in this study.
At the generated verdict position, probabilities are normalized over \texttt{SAFE} and the ten digit tokens; the total digit probability supplies the alarm score.
We use a 2,048-token budget, retaining task context and recent interactions.

\noindent \textbf{Token uncertainty.}
We compute maximum and average token entropy and negative log-likelihood under Qwen3-8B.
These scores measure a proxy model's uncertainty about the recorded text, since the logs do not provide the acting agents' own token probabilities.
They do not use failure labels to fit their scoring functions; their operating thresholds are fitted separately.

\noindent \textbf{Learned sequential scorers.}
SAFE~\citep{gu2026safe}, Hide-and-Seek~\citep{park2026hide}, and TRACES~\citep{li2026traces} use frozen encoders from the same Qwen or Llama backbones as \method{} in each comparison.
The feature at each step is extracted from its observed prefix with a 2,048-token context budget, rather than from the isolated interaction alone.
The downstream scorers are trained on the corresponding benchmark, using each method's objective and training unit.
Our SAFE adaptation uses a one-layer LSTM with hidden width 256; the Hide-and-Seek adaptation uses a one-layer GRU of the same width with its inter- and intra-trajectory losses.
SAFE uses 1,057 optimization steps on \textsc{AFTraj-2K} and 979 on StepShield, with 32 prefix rows per batch.
Hide-and-Seek uses the same step counts with batches of 32 trajectory bags.
TRACES uses 30 epochs of Stage-A pretraining, freezes that stage, and performs 528 Stage-B steps on \textsc{AFTraj-2K} or 489 on StepShield with batches of 64 prefix rows.
The reported TRACES variant includes its auxiliary prefix and ranking losses, with loss weights $1:0.2:0.05$ for the final, prefix, and ranking objectives.

\subsection{Evaluation Metrics}
\label{app:exp-metrics}

\noindent \textbf{Continuous-score metrics.}
For each test trajectory, let
\begin{equation}
    M_i=\max_{1\le t<T_i}s_{i,t}.
    \label{eq:exp-max-score}
\end{equation}
We compute AUROC from $M_i$ and $Y_i$.
FPR95 is the fraction of successful trajectories exceeding or equaling the highest empirical score threshold that attains at least 95\% recall on failed trajectories.
These two metrics are computed from the continuous trajectory scores independently of the runtime threshold band.

\noindent \textbf{Decisions and detection time.}
Let $\mathcal P=\{i:Y_i=1\}$ and $\mathcal N=\{i:Y_i=0\}$ denote the failed and successful test trajectories.
For evaluation, a trajectory is counted as detected if any nonterminal prefix receives a failure prediction, and its detection time is the first such step:
\begin{equation}
    D_i=\max_{1\le t<T_i}\widehat Y_{i,t},\qquad
    t_i^{\mathrm{det}}=\min\{t:1\le t<T_i,\ \widehat Y_{i,t}=1\}.
    \label{eq:exp-detection}
\end{equation}
The minimum is $+\infty$ when no prefix receives a failure prediction.
This first-positive summary is used to compute evaluation metrics; the auditor itself continues to issue a current prediction at each prefix.
Write $\mathrm{TP}=\{i\in\mathcal P:D_i=1\}$ and $\mathrm{TN}=\{i\in\mathcal N:D_i=0\}$.
Then $\mathrm{TPR}=|\mathrm{TP}|/|\mathcal P|$ and $\mathrm{TNR}=|\mathrm{TN}|/|\mathcal N|$.

\noindent \textbf{Accuracy and timeliness.}
Following the trajectory-level metrics used by Hide-and-Seek~\citep{park2026hide}, we compute
\begin{equation}
    \begin{aligned}
        \mathrm{wACC}
        &=\pi_0\,\mathrm{TNR}+(1-\pi_0)\,\mathrm{TPR},
        \qquad \pi_0=\frac{|\mathcal N|}{|\mathcal P|+|\mathcal N|},\\
        \mathrm{TWA}
        &=\frac12\left[\frac{1}{|\mathcal P|}
          \sum_{i\in\mathrm{TP}}\left(1-\frac{t_i^{\mathrm{det}}}{T_i}\right)
          +\mathrm{TNR}\right].
    \end{aligned}
    \label{eq:exp-accuracy}
\end{equation}
With the empirical class proportions, wACC equals the fraction of correctly classified trajectories.
TWA averages timeliness over all failed trajectories, so an undetected failure contributes zero.
Detection steps are one-indexed and normalized by the full trajectory length $T_i$, even though terminal prefixes are excluded from evaluation.
For the trade-off curves in Figure~\ref{fig:further-analysis}, we average $(t_i^{\mathrm{det}}-1)/T_i$ over failed trajectories, assigning one to each undetected failure.
Higher AUROC, wACC, and TWA are better; lower FPR95 is better.
For score-based decisions we use nominal $\alpha=0.05$; for the closed-source judges we use their returned labels, as described in Appendix~\ref{app:exp-baselines}.

\section{Supplementary Experimental Analyses}
\label{app:additional-analyses}

\subsection{Motivation Diagnostics}
\label{app:motivation-diagnostics}
We detail the trajectory-editing, representation, and learning diagnostics in Section~\ref{sec:prefix-auditor}.

\subsubsection{Trajectory Editing and Outcome Re-verification}
\label{app:trajectory-editing}

We select 128 mathematical reasoning and coding trajectories from the training pool of \textsc{AFTraj-2K} and make four edit attempts at one nonterminal action per source.
A GPT-5.4-mini editor uses the public task and observed history to propose an alternative action for the same agent, preserving the action type and tool schema without access to future interactions, outcome labels, or auditor feedback.
After replaying preceding actions and matching their observations to the recorded history, we apply the edit and resume execution with GPT-5.4-mini under a 40-message limit covering replay and continuation.
Coding outcomes are determined by running the last program submitted by \texttt{CodeWriter} against the source task's tests; these tests are not included in the continuation agents' system prompts.
Mathematical outcomes are checked against reference answers with LLM-assisted semantic review.
We retain 326 verified continuations from 97 sources with valid execution histories and resolved labels under consistent outcome checks.
We compare each continuation's outcome with its verified source label to assess label reuse after editing and re-execution.

\subsubsection{Representation and Learning Diagnostics}
\label{app:augmentation-diagnostics}

For Figure~\ref{fig:augmentation-diagnostics}, we use a frozen Qwen3-4B encoder and the verified edits described in Appendix~\ref{app:trajectory-editing}.
Panel~(a) projects 97 original--edited prefix pairs onto jointly fitted principal components; density contours use all 326 edits with equal total weight per source trajectory.
For panel~(b), we train two-layer MLP prediction heads on original prefixes alone (\emph{Original}) or the combined original and edited pool (\emph{+Edited}), using the same initialization and 1,000 AdamW updates with batch size 128 and learning rate $10^{-3}$.
We expand the original training trajectories into all their prefixes and add prefixes from the verified continuations only in the \emph{+Edited} condition.
Feature standardization and class weights are computed from the original training pool and shared by both conditions.
The plotted heads use seed 0.
With the \emph{Original} head fixed, we temporarily perturb the encoder's query and value projections (\emph{+Dream}), using per-projection rank cap 8 and relative radius $3\times10^{-4}$; the perturbation is selected on a balanced batch of 32 original training prefixes.
All conditions are evaluated on the same 2,820 held-out prefixes after excluding source groups represented in training, using the mean of the success-class and failure-class BCE losses.
The loss plane varies the head parameters along the difference between the two trained heads and scales the selected encoder perturbation independently.
The displayed coordinates are $x=v(L_D-L_O)$ and $y=u(L_E-L_O)$, where $u$ interpolates the head parameters, $v$ scales the perturbation, and $L_O,L_E,L_D$ denote the three labeled losses.
Arrow lengths reflect endpoint loss changes rather than parameter norms; dashed contours mark measurements beyond the perturbation budget.

Under the fixed training budget, adding the verified edits does not reduce held-out prediction loss.
The training-selected \emph{Dream} perturbation increases loss on the same held-out prefixes with the \emph{Original} head fixed, supporting hard latent variant construction beyond the search batch.

\subsection{Mechanism Analysis}
\label{app:mechanism-analysis}
Both experiments use \textsc{AFTraj-2K} with fixed seed-0 auditors, without updating their parameters during measurement.
The loss comparison uses the final Qwen3-8B trained with \method{}, whereas the sensitivity analysis uses the final Qwen3-4B trained without perturbations.

\subsubsection{Variant Construction}
\label{app:mechanism-construction}

\noindent \textbf{Setup.}
We sample 50 batches of 32 prefixes from the 8,454 entries in the training-prefix index described in Appendix~\ref{app:prefix-construction}.
Using sampling seed 0, we draw with replacement uniformly over parent groups, trajectories within each group, and indexed prefixes within each trajectory.
This diagnostic samples directly from the index without balancing the outcome classes.
Both perturbation strategies act on the 34 adapted query/value projections with rank cap $r=8$ and nominal relative radius $\gamma=0.003$.
The \textit{Dream} candidate is the perturbation returned by the solver in Section~\ref{sec:steer-to-confuse}.
For each random draw, we form every projection perturbation as the product of two independent standard Gaussian matrices with inner dimension eight, then jointly rescale the projection perturbations to relative norm $\gamma$.
We generate two independent random candidates per batch without selecting between them by prediction loss.

\noindent \textbf{Measurement and aggregation.}
For each candidate $V$, we measure the change in mean batch BCE relative to the unperturbed auditor:
\begin{equation}
\Delta\mathcal L_{\mathcal B}(V)
=\mathcal L_{\mathcal B}(\Theta,V)-\mathcal L_{\mathcal B}(\Theta,0).
\label{eq:mechanism-loss-change}
\end{equation}
Both losses are recomputed on the same batch in evaluation mode with gradients disabled, and all prefixes receive equal loss weight.
Each \textit{Dream} point is the returned candidate's loss change, while each random point is the median loss change of its two draws.
The summary marks in Figure~\ref{fig:loss-guided-vs-random} show the median and interquartile range, from the 25th to the 75th percentile, across the 50 points for each strategy.
The magnified inset displays the same random points, including negative changes.

\noindent \textbf{Results and interpretation.}
\textit{Dream} increases loss on all 50 batches and exceeds the corresponding random summary on each batch.
Its median loss increase is $0.0307$, compared with $6.20\times10^{-6}$ for random perturbations, whose changes occur on both sides of zero.
The returned \textit{Dream} perturbations have relative norms between $0.00261$ and $0.00280$, while random perturbations use the nominal boundary at $0.003$.
The comparison therefore evaluates the two construction strategies under the same operator scope, rank cap, and nominal budget.
The contrast supports loss-guided variant construction for a fixed auditor; the effect of learning from the resulting variants is assessed separately by the training ablation in Section~\ref{exp:abla}.

\subsubsection{Dream Sensitivity}
\label{app:mechanism-sensitivity}

\noindent \textbf{Setup.}
We hold the unperturbed Qwen3-4B fixed and construct eight \textit{Dream} perturbations from independently sampled training batches, each containing 16 successful and 16 failed prefixes.
Each solver call uses the same fixed auditor, the 34 adapted query/value projections, rank cap eight, and nominal relative radius $0.0003$.
We evaluate the returned candidates on nonterminal test prefixes after excluding every parent group represented in training and checking that no retained tokenized input exactly duplicates a training input after truncation.
This leaves 2,820 prefixes from 306 trajectories and 198 parent groups; held-out labels are not used to construct perturbations.

\noindent \textbf{Sensitivity and prediction error.}
For a held-out prefix $q=(i,t)$, let $s_q=f_\Theta(X_q)$ and let $s_q^{(k)}$ be the failure score obtained after applying the returned perturbation $\widehat V_{\mathcal B_k}$ to the same auditor.
We define its \emph{Dream sensitivity} as
\begin{equation}
H_q=\frac{1}{M}\sum_{k=1}^{M}\left|s_q^{(k)}-s_q\right|,\qquad M=8.
\label{eq:dream-sensitivity}
\end{equation}
Larger $H_q$ means that the prefix's failure score responds more strongly to the \textit{Dream} perturbations.
Prediction errors are measured for the unperturbed auditor using $\widehat Y_q=\mathbf 1\{s_q\geq\zeta_t\}$ and the recorded outcome $Y_q$.
We fit the auditor's threshold band from its saved training predictions following the procedure in Appendix~\ref{app:threshold}, with calibration seed 0 and $\alpha=0.05$, and keep this band fixed throughout the analysis.

\noindent \textbf{Grouping and uncertainty.}
We divide prefixes into five groups using the 20th, 40th, 60th, and 80th percentiles of $H_q$, yielding 564 prefixes per group.
To account for different numbers of prefixes per parent group, each prefix receives weight $w_q=1/N_{g(q)}$, where $N_{g(q)}$ is its parent's prefix count in the entire held-out pool.
For a sensitivity bin $Q_j$, the reported error rate is
\begin{equation}
\widehat e_j=\frac{\sum_{q\in Q_j}w_q\mathbf 1\{\widehat Y_q\neq Y_q\}}
{\sum_{q\in Q_j}w_q}.
\label{eq:dream-sensitivity-error}
\end{equation}
These weights are fixed before binning and normalized within each bin.
Error bars are 95\% percentile intervals from 1,000 bootstrap resamples, each drawing 198 parent groups with replacement.
All prefixes of a sampled parent inherit its resampling multiplicity, while the auditor, threshold band, perturbations, and bin boundaries remain fixed.
The intervals quantify uncertainty over held-out parent groups conditional on this fixed auditor and perturbation bank.

\noindent \textbf{Results and interpretation.}
The weighted error rate rises from $4.4\%$ in the lowest sensitivity quintile to $51.1\%$ in the highest, with increasing point estimates across all five bins.
The perturbations use only training batches, so this association connects the loss-guided construction to prediction difficulties on held-out prefixes.
These results support constructing hard latent variants that reflect weaknesses in the auditor's learned outcome discrimination beyond the construction batches.

\subsection{Additional Training Ablations}
\label{app:training-ablations}

\subsubsection{Perturbation Strategy}
The three variants in Table~\ref{tab:training-ablation-review} use training seed 0 and share the data, backbone, supervised initialization, logical batch size of 32, and optimizer-step budget.
The unperturbed variant updates the adapters and prediction head on the original batch loss.
For \emph{Random Pert.}, each logical batch receives one shared random perturbation of the adapted query/value projections.
We construct each projection perturbation as a product of two independent Gaussian matrices with inner dimension eight and jointly normalize them to $\lVert V\rVert_\omega=\gamma$, using the backbone-specific radii in Table~\ref{tab:exp-radii}.
The random candidate is used directly, without loss-based selection or backtracking, and the adapters and head are updated on the perturbed batch loss.
\emph{Dream Pert.} uses Algorithm~\ref{alg:dream2discern}; all three variants are evaluated with $V=0$.
This training comparison is distinct from the fixed-auditor diagnostic in Appendix~\ref{app:mechanism-construction}.

\subsubsection{Prefix Readout}
\label{app:readout-pooling}
To assess our default last-token readout, we compare it with mean pooling over all non-padding tokens in the final Transformer layer.
For each readout, we train a separate auditor using the same backbone, data, perturbation settings, and optimization budget, including a 100-step supervised warm start with that readout.
All runs use training seed 0.
Table~\ref{tab:readout-pooling} shows that last-token readout yields higher AUROC in all four comparisons and higher wACC and TWA with Qwen3-4B.
With Qwen3-8B, mean pooling yields higher TWA on both benchmarks and higher wACC on \textsc{AFTraj-2K}, while last-token readout retains higher wACC on StepShield.
Last-token readout therefore improves score ranking in these comparisons, while its accuracy and timeliness gains at the calibrated operating point depend on the backbone and benchmark.

\begin{table}[ht]
\centering
\small
\setlength{\tabcolsep}{4pt}
\caption{Prefix-readout comparison with training seed 0. All metrics are percentages.}
\label{tab:readout-pooling}
\begin{tabular}{llcccccc}
\toprule
& & \multicolumn{3}{c}{\textsc{AFTraj-2K}} & \multicolumn{3}{c}{StepShield} \\
\cmidrule(lr){3-5}\cmidrule(lr){6-8}
Backbone & Readout & AUROC$\uparrow$ & wACC$\uparrow$ & TWA$\uparrow$ & AUROC$\uparrow$ & wACC$\uparrow$ & TWA$\uparrow$ \\
\midrule
\multirow{2}{*}{Qwen3-4B} & Last token (default) & $95.31$ & $90.36$ & $72.28$ & $92.91$ & $90.28$ & $57.48$ \\
& Mean pooling & $91.44$ & $89.76$ & $69.56$ & $87.27$ & $79.17$ & $47.61$ \\
\midrule
\multirow{2}{*}{Qwen3-8B} & Last token (default) & $96.21$ & $91.57$ & $73.70$ & $94.59$ & $88.43$ & $56.42$ \\
& Mean pooling & $92.95$ & $92.17$ & $74.56$ & $91.98$ & $87.04$ & $57.06$ \\
\bottomrule
\end{tabular}
\end{table}

\subsection{Sensitivity to Perturbation Radius}
\label{app:radius-sensitivity}
We vary $\gamma\in\{10^{-4},3\times10^{-4},10^{-3},3\times10^{-3},10^{-2}\}$ for Qwen3-4B and Qwen3-8B on both benchmarks, keeping the remaining training and evaluation settings fixed.
All results use training seed 0.
Table~\ref{tab:radius-sensitivity} shows that larger radii do not consistently improve either metric, and the best-performing setting differs across backbones and benchmarks.
Relative to the main settings, $\gamma=10^{-2}$ lowers TWA for both backbones on \textsc{AFTraj-2K} and substantially lowers wACC on StepShield.

\begin{table}[ht]
\centering
\small
\caption{Perturbation-radius sensitivity with training seed 0. Metrics are percentages; $\dagger$ marks the radius used in the main experiments.}
\label{tab:radius-sensitivity}
\begin{tabular}{llcccc}
\toprule
\multirow{2}{*}{Backbone} & \multirow{2}{*}{$\gamma$} & \multicolumn{2}{c}{\textsc{AFTraj-2K}} & \multicolumn{2}{c}{StepShield} \\
\cmidrule(lr){3-4}\cmidrule(lr){5-6}
& & wACC$\uparrow$ & TWA$\uparrow$ & wACC$\uparrow$ & TWA$\uparrow$ \\
\midrule
\multirow{5}{*}{Qwen3-4B} & $10^{-4}$ & $91.27$ & $72.17$ & $87.50$ & $55.57$ \\
 & $3\times10^{-4}$\textsuperscript{$\dagger$} & $90.36$ & $72.28$ & $90.28$ & $57.48$ \\
 & $10^{-3}$ & $91.27$ & $73.58$ & $86.11$ & $53.72$ \\
 & $3\times10^{-3}$ & $91.87$ & $73.45$ & $80.56$ & $54.42$ \\
 & $10^{-2}$ & $89.76$ & $67.66$ & $78.70$ & $55.82$ \\
\midrule
\multirow{5}{*}{Qwen3-8B} & $10^{-4}$ & $90.06$ & $71.79$ & $90.74$ & $57.50$ \\
 & $3\times10^{-4}$ & $89.46$ & $71.83$ & $87.50$ & $55.69$ \\
 & $10^{-3}$ & $91.57$ & $74.07$ & $89.81$ & $57.04$ \\
 & $3\times10^{-3}$\textsuperscript{$\dagger$} & $91.57$ & $73.70$ & $88.43$ & $56.42$ \\
 & $10^{-2}$ & $89.46$ & $68.44$ & $62.04$ & $49.66$ \\
\bottomrule
\end{tabular}
\end{table}

\FloatBarrier
\section{Limitations}
\label{app:limitations}

Our evaluation focuses on textual execution histories and binary terminal outcomes.
Extending the auditor to visual observations or finer-grained outcome categories remains future work.
We assess prediction quality on recorded trajectories, with downstream intervention left to the deployment system.
Evaluating how auditor-guided replanning, tool restrictions, or human assistance affects task success and execution cost requires closed-loop experiments.

\end{document}